\documentclass[twoside,11pt]{article}

\usepackage{jmlr2e, amsmath, bbm, subfig, algorithm, algorithmic, amsfonts, caption, enumerate, stmaryrd, ctable, hyperref, cancel, soul, footnote, xspace, placeins}

\newcommand{\x}{{\mathbf x}}
\newcommand{\bE}{{\mathbb E}}

\newcommand{\cdalg}{\texttt{SolveRidge}\xspace}

\newcommand{\loco}{{\sc Loco}\xspace}
\newcommand{\cocoa}{{\sc CoCoA}\xspace}

\newcommand{\risk}{R}

\newcommand{\samp}{n}
\newcommand{\dims}{p}
\newcommand{\dimset}{\mathcal{P}}
\newcommand{\minusk}{(-k)}
\newcommand{\subsamp}{\samp_{subs}}
\newcommand{\blocks}{K}
\newcommand{\dimsk}{\tau}
\newcommand{\dimss}{\dims_{subs}}
\newcommand{\dimsks}{\tau_{subs}}

\newcommand{\RP}{\boldsymbol{\Pi}}
\newcommand{\feats}{\bar{\At}}

\newcommand{\soln}{\boldbeta}
\newcommand{\solnRP}{\bar{\soln}}

\newcommand{\solnls}{\solnest^{\text{ls}}}
\newcommand{\solnrls}{\widehat{\mathbf{w}}^0} 
\newcommand{\solnrstar}{\mathbf{w}^*} 
\newcommand{\solnrrr}{\widehat{\mathbf{w}}^{\text{rr}}} 
\newcommand{\Xtr}{\Xt_{\tiny{\text{R}}}}
\newcommand{\boldSigmar}{\boldSigma_{\tiny{\text{R}}}}
\newcommand{\solnrr}{\solnest^{\text{rr}}}

\newcommand{\solnloco}{\solnest^{\text{loco}}}
\newcommand{\solnest}{\widehat{\soln}}
\newcommand{\noisevar}{\mathbf{z}}

\newcommand{\boundconst}{c} 
\newcommand{\solnlsmin}{\widehat{\soln}^0}

\newcommand{\tr}{^{\top}}

\newcommand{\trace}[1]{{\rm tr}\left({#1}\right)}

\newcommand{\cb}[1]{\left\{ {#1} \right\}}
\newcommand{\br}[1]{\left( {#1} \right)}
\newcommand{\sq}[1]{\left[ {#1} \right]}
\newcommand{\nrm}[1]{\Vert {#1} \Vert}

\newcommand{\order}[1]{O \br{ #1 }}

\newcommand{\Prob}{{\mathbb P}}
\newcommand{\R}{{\mathbb R}}
\newcommand{\N}{{\mathcal N}}

\newcommand{\Xt}{{\mathbf{X}}}

\newcommand{\y}{Y}
\newcommand{\z}{\mathbf{z}}
\newcommand{\K}{\mathbf{K}}

\newcommand{\wt}{{\mathbf{w}}}

\DeclareMathOperator{\diag}{diag}

\newcommand{\At}{{\mathbf{X}}}

\newcommand{\Dt}{{\mathbf{D}}}
\newcommand{\Gt}{{\mathbf{G}}}
\newcommand{\Ht}{{\mathbf{H}}}

\newcommand{\It}{{\mathbf{I}}}

\newcommand{\Kt}{{\mathbf{K}}}
\newcommand{\Lt}{{\mathbf{L}}}

\newcommand{\Nt}{{\mathbf{N}}}

\newcommand{\St}{{\mathbf{S}}}

\newcommand{\Ut}{{\mathbf{U}}}
\newcommand{\Vt}{{\mathbf{V}}}
\newcommand{\Wt}{{\mathbf{W}}}

\newcommand{\var}{\mathbb{V}}
\newcommand{\bias}{\mathbb{B}}

\newcommand{\boldalpha}{\boldsymbol{\alpha}}
\newcommand{\boldbeta}{\boldsymbol{\beta}}

\newcommand{\boldSigma}{\boldsymbol{\Sigma}}

\newtheorem{thm}{Theorem}
\newtheorem{cor}[thm]{Corollary}

\newtheorem{lem}[thm]{Lemma}
\newtheorem{rem}{Remark}
\newtheorem{defn}{Definition}
\newtheorem{assn}{Assumption}

\title{{\sc Loco}: Distributing Ridge Regression with Random Projections}

\author{\name Christina Heinze \email heinze@stat.math.ethz.ch \\
       \addr Seminar for Statistics\\
       ETH Z\"urich\\
       8092 Z\"urich, Switzerland
       \AND
	   \name Brian McWilliams \email brian.mcwilliams@inf.ethz.ch \\
       \addr Department of Computer Science\\
       ETH Z\"urich\\
       8092 Z\"urich, Switzerland
       \AND
       \name Nicolai Meinshausen \email meinshausen@stat.math.ethz.ch \\
       \addr Seminar for Statistics\\
       ETH Z\"urich\\
       8092 Z\"urich, Switzerland
       \AND
       \name Gabriel Krummenacher \email gabriel.krummenacher@inf.ethz.ch \\
       \addr Department of Computer Science\\
       ETH Z\"urich\\
       8092 Z\"urich, Switzerland
       }

\editor{...}

\renewcommand{\algorithmicrequire}{\textbf{Input:}}
\renewcommand{\algorithmicensure}{\textbf{Output:}}

\jmlrheading{x}{20xx}{x-xx}{05/15}{xx/xx}{Christina Heinze, Brian McWilliams, Nicolai Meinshausen, 
Gabriel Krummenacher}

\ShortHeadings{{\sc Loco}: Distributing Ridge Regression}{Heinze, McWilliams, Meinshausen and Krummenacher}
\firstpageno{1}

\begin{document}

\maketitle

\begin{abstract}
We propose \loco, an algorithm for large-scale ridge regression which distributes the features across workers on a cluster. Important dependencies between variables are preserved using structured random projections which are cheap to compute and must only be communicated once. We show that \loco obtains a solution which is close to the exact ridge regression solution in the fixed design setting. We verify this experimentally in a simulation study as well as an application to climate prediction. Furthermore, we show that \loco achieves significant speedups compared with a state-of-the-art distributed algorithm on a large-scale regression problem. 
\end{abstract}


\section{Introduction}
In the last few years there has been great interest in solving large-scale optimization and estimation problems. Parallelization has naturally emerged to leverage now-commonplace multi-core architectures to enable moderate sized problems to be solved quickly. Some datasets are large enough such that they are impractical to store and process on a single machine and so the problem must be solved in a distributed manner on a computing cluster. 

Two obvious questions arise: (1) how should the data and processing tasks be distributed among processing units (workers) and (2) how and what should each worker communicate. The data, computing architecture and choice of learning algorithm influence both of these points. 
Stochastic gradient descent (SGD) methods are suited to parallelization over the rows (observations) of the data \citep{zinkevich2010parallelized}. However, synchronization of results to ensure each worker is updating the current gradient becomes expensive. This has motivated recent asynchronous  approaches to parallel SGD  \citep{hogwild,Duchi:2013}. 

In this work, we limit our focus to $\ell_2$ penalized linear regression for large-scale estimation tasks -- in particular when the dimension $\dims$ of the data is also large. In such cases, it may not be practical to process the entire dataset on a single multi-core machine due to memory limitations.  
We therefore wish to distribute the problem in a way which allows computation to be shared across many machines which do not share memory. Common approaches outlined above typically limit the number of \emph{samples} each worker sees but when $\dims$ is large, it is preferable for each worker to instead solve a lower dimensional problem. This setting motivates distributing the data according to \emph{features} rather than samples. Distributing the task in this way introduces an additional difficulty -- features are not assumed to be independent as samples are. Therefore, care must be taken to maintain the important dependencies between features while keeping synchronization and communication between workers at a minimum. Another natural setting which motivates distributing estimation across features is one of privacy preservation. In this framework, no single worker may see all of the features and so even when the data size is not massive, sharing memory and data between workers is not permitted. 

Randomized dimensionality reduction based on the Johnson-Lindenstrauss lemma has emerged as a way to quickly obtain good approximations to a variety of learning tasks \citep{Ailon:2009}. Notably, structured random projections have been used to speed up approximate kernel expansions \citep{Le:2013}, computation of statistical leverage scores \citep{Drineas:2011ts,mcwilliams2014fast} and linear regression. For the latter, it can be shown that the least squares solution computed on a random projection of either the row \citep{Mahoney:2011te} or column  \citep{lu:2013, kaban:2014} space of the data matrix results in a solution which is close to optimal. An obvious downside to dimensionality reduction is that the solution obtained is no longer in the original space. Therefore, the estimated coefficients are difficult to interpret with respect to the observed features -- a task often as important as prediction accuracy. Furthermore, in order to compute the projection, a single machine is assumed to have access to the entire dataset.

In this work we propose and analyze \loco \xspace -- a simple, {\sc lo}w-{\sc co}mmunication distributed algorithm to approximately solve $\ell_2$ penalized least squares estimation which crucially requires no synchronization between workers. \loco assigns features to workers by randomly partitioning the data into $\blocks$ blocks (alternatively, this may be part of the problem specification). In each block, a small number of cheaply computed random projections are used to approximate the contribution from the remaining columns of the data. This ensures that important dependencies between features are maintained. Each worker then simply optimizes the objective independently on this compressed dataset, the size of which is proportional to the size of the random projection and the total number of workers. The solution vector returned by \loco is constructed by collecting the estimates for the respective unprojected ``raw'' features from each worker such that it lies in the original space. 

\loco is particularly suited to the high-dimensional setting (i.e.\ when the number of dimensions $\dims$ is larger than the number of samples $\samp$). High-dimensional data occurs frequently in practice: for example in bioinformatics, climate science and computer vision among others. Furthermore, it is often desirable to expand the dimensionality of low-dimensional data to improve predictions using e.g.\ higher-order interactions, feature transformations or representation learning. In such high-dimensional settings \loco retains good statistical properties and benefits from large potential speedups.

\paragraph{Outline and Contribution.} 

In \S\ref{sec:related} we place our contribution in the context of recently proposed related approaches to distributed optimization. In \S\ref{sec:setting} we formally describe our estimation problem and the distributed setting which we consider. We also give a brief introduction to random projections, in particular the Subsampled Randomized Hadamard Transform (SRHT) \citep{Tropp:2010uo}. In \S\ref{sec:algo} we describe \loco, our algorithm for distributed ridge regression. In \S\ref{sec:analysis} we show in the fixed design setting that the error between the coefficients estimated by \loco and the optimal ridge regression coefficients is bounded, under natural assumptions about the problem setting -- that some proportion of the signal lies in the top principal components. Importantly, unlike other approaches to parallelizing or distributing optimization, we make \emph{no} assumptions on sparsity in the data.   In \S\ref{sec:results} we provide implementation details and empirical evaluation of our algorithm on large-scale simulated and real datasets. \loco typically exhibits significant speedups with the number of workers with little loss in prediction accuracy.

\section{Related work} \label{sec:related}

Recently a number of methods have been proposed for large-scale optimization which parallelize the problem either locally amongst multiple cores on the same physical machine with shared memory or in a distributed fashion on a computing cluster. These general approaches are not mutually exclusive and can often be used in combination. However, each has a domain for which it is best suited as well as its own specific drawbacks.

Here we will briefly review some of the main directions in parallel and distributed optimization and estimation. 

\paragraph{Parallel methods.}
Parallel methods such as {\sc hogwild!}  \citep{hogwild}, {\sc AsyncDA} and {\sc AsyncAdaGrad}  \citep{Duchi:2013} have shown that large speedups are possible with asynchronous gradient updates when data is sparse. These methods rely on the idea that if the number of non-zero coordinates in each stochastic gradient evaluation is small compared to the number of variables $\dims$, workers updating the same solution vector in parallel will rarely propose conflicting updates. As such each worker is allowed to update the solution asynchronously without the need for locking, provided the delay of any processor is not too great.

Whilst sparsity is a natural and common feature of large datasets, in some fields the data collected is dense with many correlated features. Furthermore, in the high-dimensional setting, SGD in particular may take many passes over the data to reach the optimum. Under these conditions we might expect the performance of the above mentioned approaches to suffer. 

Local parallel methods are able to achieve large speedups in part because the data is assumed to be stored locally and each core is able to access shared memory which makes communication relatively cheap. As dataset size increases several limiting factors arise: the number of processors on a single machine, the amount of local memory,  and finally the local storage size. In other cases, the dataset might be physically stored in several different locations. These aspects can make parallel optimization impractical for particularly large scale problems.

\paragraph{Distributed methods.}
Methods which distribute computation amongst $\blocks$ networked workers on a cluster have been proposed which alleviate these constraints. However, the communication between workers introduces significant overhead -- it can be orders of magnitude slower than accessing local memory. This necessitates a different update strategy from the parallel approaches outlined above. 

\citet{jaggi2014communication} propose a communication efficient approach to dual optimization ({\sc CoCoA}). In each iteration, each worker solves a local dual problem, using a fraction of $\samp/\blocks$ samples, and communicates the coefficient estimates which are then aggregated. This procedure is iterated until convergence. The user can steer the tradeoff between communication and local computation by specifying how many data points to process locally in each iteration. 

Aside from these methods which consider quite broad classes of optimization problems, several methods have been proposed for solving specific statistical estimation tasks in a distributed fashion. For example, \citet{zhang:2013} considered the problem of kernel ridge regression. Each worker computes a local estimator using $\samp/\blocks$ samples which is then communicated back to the master.  Since only a single round of communication is necessary, simply computing an estimate which is the average of the local estimates achieves a superlinear speedup whilst retaining an optimal rate of convergence (in the statistical sense) up to a number of workers, $\blocks$ which is problem dependent. \citet{liu2014distributed} address the more general problem of distributed maximum likelihood estimation in exponential family models.

Each of these methods for optimization and estimation analyzes strategies which distribute across the \emph{samples} only. The setting where each worker receives a subset of the \emph{features} has received less attention.  \citet{Richtarik:2013} proposed a distributed approach to {coordinate descent} where each worker sees a block of features. Similarly to parallel approaches to coordinate descent \citep{Bradley:2011}, blocks of features are required to be nearly independent to keep communication costs down. 

In this work we focus on the setting where each worker receives a subset of the features. Notably however, \loco requires no assumptions about sparsity or independence between features since each block sees a representation of the remaining features such that updates to the individual solution vectors are not independent of the rest of the dataset. 
\loco does not require synchronization between workers since each worker may only update its own part of the solution vector.

\paragraph{Johnson-Lindenstrauss projections.} Johnson-Lindenstrauss (J-L) projections are a popular method for dimensionality reduction. J-L projections are low-dimensional embeddings which preserve -- up to a small distortion -- pairwise $\ell_2$ distances between vectors according to the J-L lemma (see e.g.\ \citet{Ailon:2009}). Specific constructions also guarantee that the spectrum of an entire subspace of vectors is preserved \citep{Tropp:2010uo}.  Typically, the projection matrix is constructed to be a nearly-orthogonal matrix with entries drawn at random from a sub-gaussian distribution \citep{achlioptas2003}. Recently, fast constructions based on sparse matrices \citep{Ailon:2009}, or highly structured matrices \citep{Halko:2011kg, Boutsidis:2012tv} have been proposed which retain similar guarantees but reduce the dependence of the computational cost on the dimension from linear to logarithmic.

Random projections have been used for dimensionality reduction for least squares \citep{kaban:2014} and ridge regression \citep{lu:2013}. However, the solution vector is in the compressed space and so interpretability of coefficients is lost. 

\section{Problem Setting and Notation} \label{sec:setting}
In this work we will concentrate on ridge regression, a ubiquitous tool for high-dimensional data analysis \citep{esl}.
Given a matrix of features $\At\in\R^{\samp\times\dims}$ and a corresponding vector of responses, $\y\in\R^\samp$ where the dimensionality $\dims$ and sample size $\samp$ are very large, we are interested in solving the following estimation task
\begin{equation}
\min_{\soln \in\R^{\dims}} L(\soln) := \samp^{-1} \nrm{\y - \At\soln}^2 +  \lambda\nrm{\soln}^2 
\label{eq:obfn}
\end{equation}
The first term is the squared error loss and the second term is the ridge penalty which regularizes the size of the coefficient vector according to the tuning parameter, $\lambda$.\footnote{Throughout, $\nrm{\cdot}$ refers to the Euclidean norm for vectors and the spectral norm for matrices, i.e.\ $\nrm{\mathbf A }=\sup_{\x}\nrm{\mathbf{A}\x}/\nrm{\x}$.} 


Ridge regression has a closed-form solution
$\solnrr = \br{\At\tr\At + \samp \lambda \It }^{-1}\At\tr\y$, but
clearly when the dimensionality of the data is large, constructing and inverting the covariance matrix is prohibitively expensive. 
When the number of samples is very large, ridge regression is usually solved using stochastic gradient descent (SGD) or stochastic dual  coordinate ascent (SDCA) \citep{shalev:2013}.


\paragraph{Feature-wise distributed ridge regression.} We now consider the case where we distribute the features across $K$ different workers. Formally, let $\dimset = \lbrace 1,\ldots,\dims \rbrace$ be the set of indices. We partition this set into $\blocks$ non-overlapping subsets $\dimset_{1},\ldots,\dimset_{\blocks}$ of equal size, $\dimsk=\dims/\blocks$ so $\dimset = \bigcup_{k=1}^\blocks  \dimset_{k}$ and $|\dimset_1|=|\dimset_2|,\ldots,=|\dimset_{\blocks}|=\dimsk$.\footnote{This is for simplicity of notation only, in general the partitions can be of different sizes.} 

A naive attempt at parallelizing \eqref{eq:obfn} would simply be solving the minimization problem on each subset of features $\dimset_k$ independently. However, without sparsity in the dataset to guide the partitioning process, important dependencies between features in different blocks would not in general be preserved. 

We can rewrite \eqref{eq:obfn} making explicit the contribution from block $k$. 
Letting $\At_{k}\in\R^{\samp\times \dimsk}$ be the sub-matrix whose columns correspond to the coordinates in $\dimset_{k}$ (the ``raw'' features of block $k$) and $\At_{\minusk}\in\R^{\samp\times (\dims-\dimsk)}$ be the remaining columns of $\At$, we have
\begin{equation} \label{eq:optim_global}
L(\soln) = \samp^{-1} \nrm{\y - \At_k\soln_{\text{raw}} - \At_{\minusk}  \soln_{\minusk}}^2 
+ \lambda  \nrm{ \soln_{\text{raw}}}^2 + \lambda \nrm{\soln_{\minusk}}^2. 
\end{equation}
The idea behind our approach is to replace $\At_{\minusk}$ in each block with a low-dimensional approximation. Since the regularizer is separable across blocks, we only require that the contribution from $\At_{\minusk}  \soln_{\minusk}$ to $f(\soln)$ is preserved.



\begin{figure*}[!tp]
\begin{centering}
    \includegraphics[width=0.65\textwidth, keepaspectratio=true]{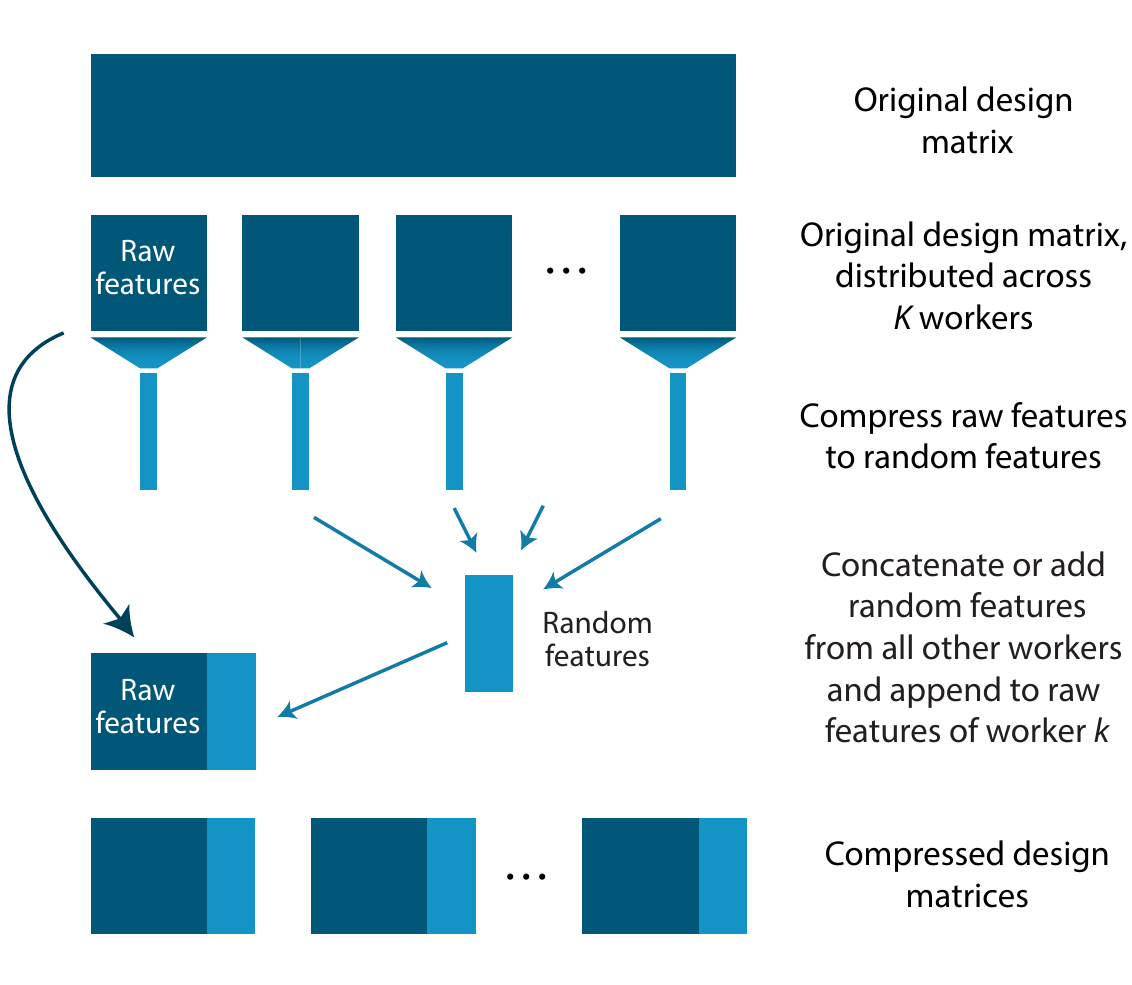}
\caption{Schematic for the approximation of a large data set in a distributed fashion using random projections. The random features can either be concatenated or added. \label{fig:loco_partition}
}
\end{centering}
\end{figure*}

Let $\tilde{\At}_k\in\R^{\samp\times (\blocks-1)\dimsks}$ be the matrix whose columns are a low-dimensional approximation to $\At_{\minusk}$, i.e.\ to the columns of $\At$ not in $\At_k$,  and $\dimsks\ll \dimsk$. The matrix $\tilde{\At}_k$ approximates $\At_{\minusk}$ as each of the other $\blocks-1$ blocks provides an approximation to its respective $\dimsk$ raw features of size $\dimsks$, resulting in $(\blocks-1) \dimsks$ columns. We shall call the columns in $\tilde{\Xt}_k$ the ``random'' features of block $k$. This procedure is described in Figure \ref{fig:loco_partition}.

Defining the sub-problem that worker $k$ solves as
\begin{equation}\label{eq:optim_workerk}
L_k(\soln_k) = \samp^{-1} \nrm{\y - \underbrace{\At_k\soln_{\text{raw}}}_{\text{raw}} - \underbrace{\tilde{\At}_k \soln_{k, {\text{rp}}}}_{\text{random}}}^2 + \lambda  \nrm{ \soln_{\text{raw}}}^2 + \lambda \nrm{\soln_{k, \text{rp}}}^2,
\end{equation}
we require the approximation $\tilde{\At}_k$ to be such that the risk of the estimator which minimizes eq.~\eqref{eq:optim_workerk} is similar to the risk of the minimizer of eq.\ \eqref{eq:optim_global} (we formalize this in \S \ref{sec:analysis}). In order to achieve this we construct the approximation using random projections which we briefly describe below.



\paragraph{Subsampled Randomized Hadamard Transform.}  J-L projections are low-dimensional embeddings $\RP :\R^\dimsk\rightarrow \R^{\dimsks}$.
We concentrate on the class of \emph{structured} random projections, among which the Subsampled Randomized Hadamard Transform (SRHT) has received particular recent attention \citep{Tropp:2010uo,Boutsidis:2012tv}. 
The SHRT consists of a preconditioning step after which $\dimsks$ columns of the new matrix are subsampled
uniformly at random. In more detail, it consists of a projection matrix, $\RP\in\R^{\dimsk\times\dimsks}= \sqrt{\dimsk/\dimsks} \Dt\Ht\St$ 
\citep{Halko:2011kg,Boutsidis:2012tv} with the definitions:
\begin{itemize}
\item $\St\in\R^{\dimsk\times\dimsks}$ is a subsampling matrix.
\item $\Dt\in\R^{\dimsk\times\dimsk}$ is a diagonal matrix whose entries are drawn independently
  from $\{-1, 1\}$. 
\item $\Ht \in \R^{\dimsk \times \dimsk}$ is a normalized Walsh-Hadamard
  matrix\footnote{For the Hadamard transform, $\dimsk$ must be a power
    of two but other transforms exist (e.g. DCT, DFT)
    with similar theoretical guarantees and no
    restriction on $\dimsk$.} which is defined recursively as
$$
\Ht_\dimsk = \sq{\begin{array}{cc} \Ht_{\dimsk/2} & \Ht_{\dimsk/2} \\ \Ht_{\dimsk/2} & -\Ht_{\dimsk/2} \end{array} }
,~~
\Ht_2 = \sq{\begin{array}{cc} +1 & +1 \\ +1 & -1 \end{array} }.
$$
We set $\Ht = \frac{1}{\sqrt{\dimsk}}\Ht_\dimsk$ so it has orthonormal
columns.
\end{itemize}

The SRHT has similar $\ell_2$ distance preserving properties as sub-gaussian random projections but has the added benefit of a fast $\order{\dimsk\log\dimsk}$ matrix-vector product due to its recursive definition. 


\section{Algorithm}\label{sec:algo}
Our procedure \loco for distributed ridge regression is presented in Algorithm \ref{alg:batch}. We describe the steps in more detail below.

\begin{algorithm}[ht]
\caption{\loco  \label{alg:batch}}
	\algorithmicrequire\; Data: $\At$, $\y$, Number of blocks: $\blocks$, Parameters: $\dimsks$, $\lambda$
  \begin{algorithmic}[1]
  \STATE Partition $\dimset=\{1,\ldots,\dims\}$ into $\blocks$ subsets $\dimset_{1},\ldots,\dimset_{\blocks}$ of equal size, $\dimsk$.
    \FOR{\textbf{each} worker $k\in\{1,\ldots\blocks\}$ \textbf{in parallel}}
        \STATE Compute and send random projection $\widehat{\At}_k = \At_k \RP_k$. 
        \STATE Construct $\feats_k = [\At_k, \tilde{\At}_k] $ 
        \STATE $\solnRP_k \leftarrow \cdalg(\feats_k,\y,\lambda) $
        \STATE $\solnest_k = \sq{\solnRP_k}_{1:\dimsk}$
    \ENDFOR
	  \end{algorithmic}
  \algorithmicensure\; Solution vector: $\solnloco = \sq{\solnest_1,\ldots,\solnest_\blocks}$
\end{algorithm} 
\paragraph{Input.} As well as the usual regularization parameter $\lambda$, \loco requires the specification of the number of workers $\blocks$ and the random projection dimension $\dimsks$.
\paragraph{Steps 1 \& 3.} We first randomly partition the coordinates into $\blocks$ subsets. Then each worker computes a random projection, via the SRHT, of its respective block which we denote by 
$\widehat{\At}_k = \At_k \RP_k \in \R^{\samp\times\dimsks} $.
 
\paragraph{Step 4.} Each worker $k$ constructs the matrix 
$$
\feats_k\in\R^{\samp\times(\dimsk+ (\blocks-1)\dimsks)} = \sq{\At_k, \tilde{\At}_k}, \quad \tilde{\At}_k = \sq{\widehat{\At}_{k'}}_{k'\neq k}
$$
which is the column-wise concatenation of the raw feature matrix $\At_k$ and the random approximations from all other blocks, $\tilde{\At}_k$. 

\paragraph{Alternative Step 4.}
Each worker $k$ constructs the matrix 
$$
\feats_k\in\R^{\samp\times(\dimsk + (\blocks-1)\dimsks)} = \sq{\At_k, \tilde{\At}_k}, \quad \tilde{\At}_k = \sum_{k'\neq k} \widehat{\At}_{k'} .
$$
When $\RP$ is defined explicitly as a random matrix (e.g.\ entries sampled i.i.d.\ from a sub-Gaussian or very sparse distribution), summing $\R^\dimsk\rightarrow\R^{(\blocks-1)\dimsks}$-dimensional random projections from $(\blocks-1)$ blocks is equivalent to computing the $\R^{(\dims - \dimsk)}\rightarrow\R^{(\blocks-1)\dimsks}$-dimensional random projection in one go which potentially allows for the random feature representation to be computed and combined more efficiently. \newline 

Without loss of generality the raw features will always occupy the first $\dimsk$ columns of $\feats_k$. The last $(\blocks-1)\dimsks$ columns of $\feats_k$ are a good approximation of the remaining $(\blocks-1)$ blocks of the full data matrix not in $\At_k$ and so solving \eqref{eq:obfn} using $\feats_k$ obtains a solution which is close to the optimal solution using $\At$. We make this explicit in \S\ref{sec:analysis}.

\paragraph{Steps 5 \& 6.} The function $\cdalg(\feats_k,\y,\lambda)$ returns a vector 
\begin{equation}
\solnRP_k\in\R^{\dimsk + (\blocks-1)\dimsks} = \arg\min_{\soln_k} \samp^{-1} \nrm{\y - \feats_k \soln_k}^2 + \lambda\nrm{\soln_k}^2 \label{eq:ridge_k}
\end{equation}
In practice, any fast algorithm which returns an accurate solution to eq. \eqref{eq:ridge_k} can be used here. The final solution vector $\solnloco$ is the concatenation of the first $\dimsk$ coordinates of each $\solnRP_k$ and so lives in the same space as the original data.

\paragraph{Computational, memory and communication costs.} Each worker must only store a copy of its block of raw features and a random projection of the remaining features which is $\order{\dimsk+(\blocks-1)\dimsks}$. This is substantially smaller than the original dimensionality $\dims$. Each worker must communicate its random projection once to all other workers (or to a shared location where the other workers can read  it). Aside from this there is no further communication between workers.  The small size of the projection ensures that for appropriately sized problems, each worker is able to store its relevant features in local memory.

A key benefit of \loco which differentiates it from most other distributed algorithms is that there are three areas where speedups are possible as $\blocks$ increases. 
\begin{enumerate}
	\item[(i)] {\bf The problem each worker solves becomes easier in a computational sense.} The cost of computing a fast random projection of the $\dimsk$ features in each block is $\order{\samp\dimsk\log\dimsks}$. As $\blocks$ increases $\dimsk$ decreases, resulting in a speedup in the computation of the random projection. As long as the total number of features per worker $(\dimsk + (K-1)\dimsks)$ also decreases, each iteration of the local optimization algorithm becomes cheaper. Assuming a solver whose iteration cost scales linearly with the problem dimension is used in $\cdalg(\feats_k,\y,\lambda)$,
the part of the computational cost which is dependent on the dimension scales with $\order{\dimsk \log\dimsks + \dimsk + (\blocks-1)\dimsks}$. 

\item[(ii)] {\bf Each local problem becomes easier in a statistical sense}. The ratio between the number of parameters and the sample size $(\dimsk + (K-1)\dimsks)/\samp$ decreases allowing faster convergence to the optimal solution in each block. 

\item[(iii)] As a consequence of (i), {\bf the size of the random projections to be communicated by each worker decreases}. 
\end{enumerate}

The speedup occurring from point (i) is common to all distributed algorithms. However, the speedup contribution from (ii) and (iii) are specific to \loco. In contrast, row-wise distribution often involves a trade-off between speed increases coming from (i) and a \emph{slow-down} coming from the fact that the local problems are \emph{more} high-dimensional and so local convergence will be slower.


\section{Analysis} \label{sec:analysis}

In this section we will show that in the fixed design setting the coefficients estimated by \loco are close to the full ridge regression solution. The results here are developed for the case where the random features in $\tilde{\Xt}_k$ result from concatenating the SRHT projections of all other blocks and throughout we shall assume that the columns of $\Xt$ and $\feats_k$ are standardized. \linebreak
 
\noindent Consider the linear model 
\begin{equation}
 \y = \At \soln^*  + \varepsilon,
\label{eq:linear_model} 
\end{equation}
with fixed $\At \in \mathbb{R}^{n \times p}$ and  true parameter vector $\soln^* \in \mathbb{R}^p$. Assumption \ref{assn:bound_t} below will formalize our assumptions on the noise, $\varepsilon$. Let $\solnrr$ denote the ridge estimate for $\soln^*$, so $\solnrr$ is the solution which results from solving the ridge regression problem in the original space, stated in eq. \eqref{eq:optim_global}.

In order to formulate our result, we define the following risk function. 
\begin{defn}[Risk]
Let  $\widehat{\mathbf{b}}$ be an estimator for $\soln^*$ and define the risk of $\widehat{\mathbf{b}}$ with fitted values $\widehat{\y} = \Xt\widehat{\mathbf{b}}\in\R^n$ as
$$
\risk(\Xt\widehat{\mathbf{b}}) = \samp^{-1} \bE_\varepsilon \nrm{\Xt \soln^* - \Xt\widehat{\mathbf{b}}}^2.
$$
\end{defn}

In the formulation of Theorem \ref{thm_ridge} we make use of the fact that we can rewrite the regularized optimization problems in eqs.\ \eqref{eq:optim_global} and \eqref{eq:optim_workerk} as constrained optimization problems with a monotonic relationship between the regularization parameter $\lambda$ and the constraint $t$ which upper-bounds the squared $\ell_2$ norm of the solution vector. In the original space we have
\begin{equation}
\underset{\nrm{\soln}^2 \leq t}{\min} n^{-1} \nrm{\y - \At \soln}^2 
\label{analysis:global_ridge}
\end{equation}
while each worker solves
\begin{equation}
\underset{\nrm{\soln_k}^2 \leq t}{\min} n^{-1} \nrm{\y - \feats_k \soln_k}^2 
\label{analysis:workerk_ridge}
\end{equation}
in a compressed space. Recall that $\solnrr $ minimizes eq.~\eqref{analysis:global_ridge} and $\solnRP_k$ minimizes eq.~\eqref{analysis:workerk_ridge}. \linebreak

\noindent Before we state our main theorem, we make the natural assumption that the main contribution to the $\ell_2$ norm of the true parameter vector -- i.e. most of the important signal -- lies in the direction of the first $J$ principal components of $\Xt$. This merely formalizes the conditions under which ridge regression yields good results. Since ridge regression applies more shrinkage in directions associated with smaller eigenvalues \citep{esl},  if Assumption \ref{assn:bound_t} does not hold we might expect a different estimator to be more appropriate.

\begin{assn}
Let $\solnrstar$ be the true parameter vector after rotating $\Xt$ to the PCA coordinate system.
There exists $1\le J\le \min\{n,p\}$ and $\boundconst\in (0,1/2)$ such that 
\begin{enumerate}
\item[(A1)] the $J$-th largest eigenvalue of the covariance matrix is strictly positive, that is $\lambda_J >0$,
\item[(A2)] the ridge constraint is active: $t\le (1-\boundconst) \sum_{j=1}^J (\solnrstar_j)^2  $,
\item[(A3)] the errors $\varepsilon_i$, $i=1,\ldots,n$ have zero mean, are independent and their variances are bounded by $\sigma^2>0$.
\end{enumerate}
\label{assn:bound_t}
\end{assn}

To shed some light onto Assumption (A2), consider the noiseless case where the entire signal lies in the first $J$ principal components. Then $c = 0$ implies no shrinkage, while increasing $c$ means that the amount of regularization becomes larger. 

If Assumptions (A1) and (A2) do not hold, then ridge regression may not be a suitable estimator for $\soln^*$ in Eq. \eqref{eq:linear_model}, independent of how we choose the size of the constraint. If, on the other hand, (A1) and (A2) do hold, the amount of required regularization can differ. In problem settings where the signal-to-noise ratio is low,
$p>n$ or where the covariance matrix of $\Xt$ is otherwise close to singular, the ridge constraint is active and (A2) covers the relevant section of the regularization parameter. The ridge estimator will then shrink less along directions associated with large variance.
If the data are full rank and the noise value is very low, shrinkage may be unnecessary and the ordinary least squares estimator may be more appropriate. This issue is  discussed in \S\ref{sec:supp-rr} and we derive a similar bound for OLS in \S \ref{sec:supp_ols}, which is the relevant bound if the ridge constraint is not active and (A2) does not apply.

We now present Theorem \ref{thm_ridge} which states that the expected difference between the coefficients $\solnloco$ returned by \loco and the full ridge regression solution is bounded. 
\begin{thm}

Under Assumption \ref{assn:bound_t}, $ \exists ~ n_0(\xi) $ for all $\xi > K (\delta +(p-\tau)/e^r)$ such that for all $n \geq n_0$ with probability at least $1-\xi$
$$
 \bE_\varepsilon[\nrm{\solnrr -\solnloco}^2] \leq 
\frac{5 \blocks}{\boundconst \lambda_J} \br{\frac{1}{(1-\rho)^{2}} -1} \risk(\Xt\solnrr)
$$
where $\rho = C \sqrt{\frac{r \log(2r / \delta) }{(\blocks-1)\dimsks}}$, $r=\text{rank}(\Xt)$, $\lambda_J$ denotes the $J^{th}$ largest non-zero eigenvalue of the covariance matrix and
$\risk(\Xt\solnrr)$ is the risk of the ridge estimator. The expectation is conditional on the random projection as the uncertainty coming from the SRHT is captured in the probability with which the statement holds. 
\label{thm_ridge}
\end{thm}
The exact value of $\samp_0$ depends on $\xi$ and the exact form is given in the proof of Theorem \ref{thm_ridge} which is presented in \S\ref{sec:supp-rr}.

The bound above intuitively trades off several fundamental quantities which determine the overall approximation error; the projection dimension, $\dimsks$, the number of workers, $\blocks$ and the rank of the data, $r$. The bound scales with the number of workers and inversely with $(1-\rho)^2$, which measures the quality of the random feature representation. This can be improved by either increasing the projection dimension, $\dimsks$, for a fixed $\blocks$ or by increasing the number of workers, $\blocks$, for a fixed $\dimsks$. However, doing so increases the computational overhead per worker which scales as $\order{(\blocks-1)\dimsks}$.  

The approximation quality term $\rho$ also depends on the rank, $r$ of the design matrix. Intuitively, if $r$ is small the performance of \loco improves. For a fixed projection dimension the random feature representation is most successful in capturing the signal of its corresponding block of raw features if the rank of that block is not too large. If $\blocks$ and $\dimsks$ are chosen such that $(\blocks-1) \dimsks >> r$, the approximation error vanishes.


\section{Experimental Results} \label{sec:results}

\paragraph{Implementation details.} 
We implemented \loco in the Apache Spark framework and ran the experiments on the Brutus and Euler clusters\footnote{\url{http://en.wikipedia.org/wiki/Brutus_cluster}}. A software library  {\sc Loco}$^{\text{lib}}$ is available at \url{http://christinaheinze.github.io/loco-lib/}.
In practice, to guarantee portability across different computing architectures, instead of using specialized libraries providing the SRHT, we used a sparse random projection matrix \citep{achlioptas2003} with entries sampled as
$$
\RP_{i,j} \sim \left\{ \begin{array}{cc}1, & \text{ w.p. ~} 1/6 \\ 0, & \text{ w.p. ~} 2/3\\ -1, & \text{ w.p. ~} 1/6  \end{array} \right. .
$$   
The sparse random projection matrix has similar guarantees to the SRHT. However, in the case of fully dense data it is not as fast. In the future we aim to add SRHT functionality to {\sc Loco}$^{\text{lib}}$. The local ridge regression solver called by \cdalg is SDCA \citep{shalev:2013}.
We use the alternative step 4 in the algorithm as it allows for a more efficient aggregation of the random projection. Although Theorem \ref{thm_ridge} applies only for concatenating random projections, we find that summing also performs well in practice.

\paragraph{Competing methods.}
We compared against \cocoa \citep{jaggi2014communication} which is also implemented in Spark\footnote{Code available from: \url{https://github.com/gingsmith/cocoa/}} and ran it on the same cluster. We modified the local solver in \cocoa to ridge regression, also using SDCA.

\begin{figure*}[!tp]
\begin{centering}
\subfloat[\xspace]{
    \includegraphics[trim=0 40 0 0, clip, width=0.48\textwidth, keepaspectratio=true]{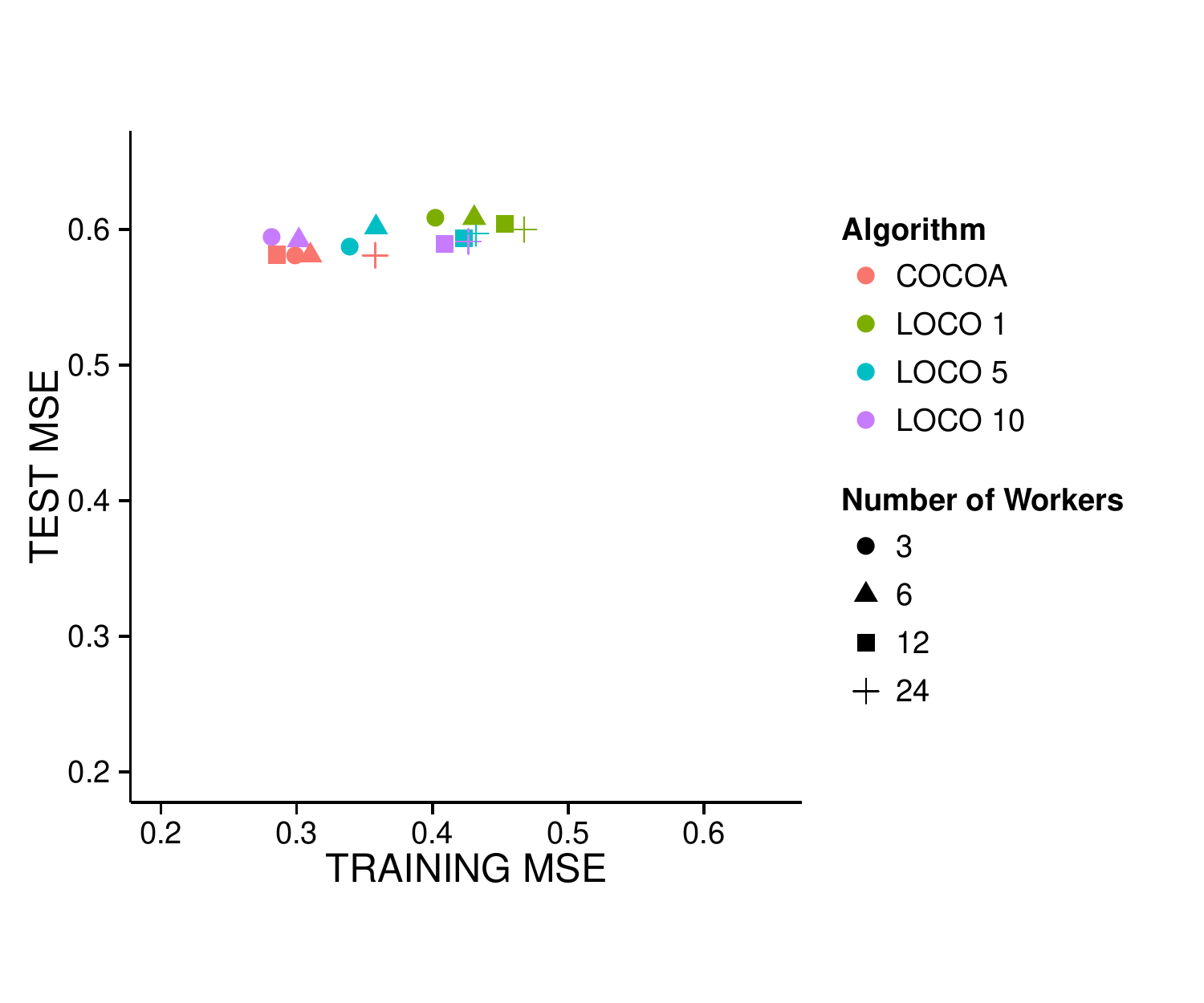}
    \label{fig:error_150K}
}
\hspace{-0.5cm}
\subfloat[\xspace]{
\vspace{1cm}
    \includegraphics[trim=0 51 0 0, clip, width=0.48\textwidth, keepaspectratio=true]{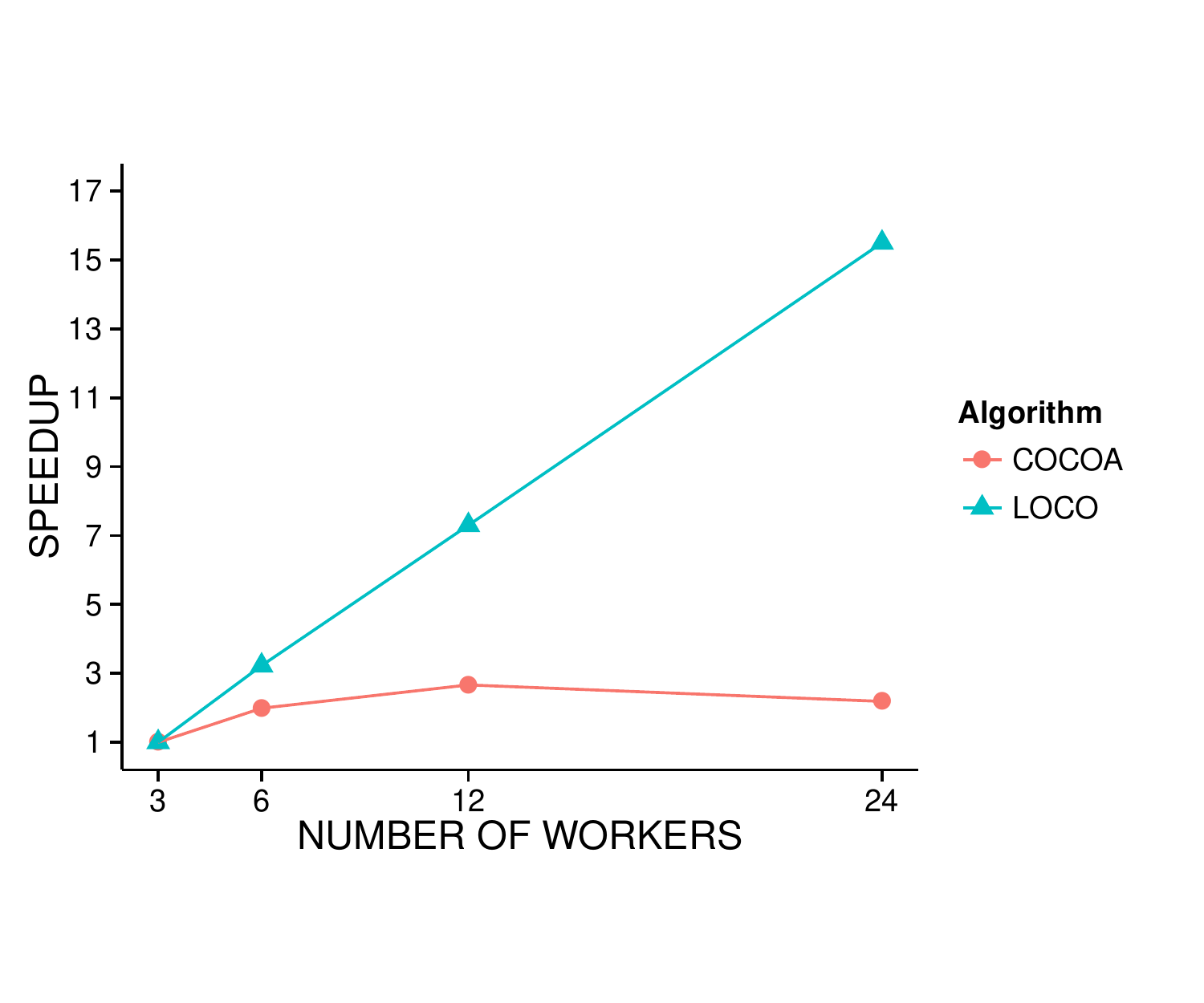}
    \label{fig:speedup_150K}
}
\vspace{-5pt}
\caption{\protect\subref{fig:error_150K} Normalized training and test error and \protect\subref{fig:speedup_150K} relative speedup for different number of workers when $\dims=150,000$. \label{fig:150K}}
\end{centering}
\vspace{-0.3cm}
\end{figure*}

\begin{figure*}[!tp]
\begin{centering}
\subfloat[\xspace]{
    \includegraphics[width=0.48\textwidth, keepaspectratio=true]{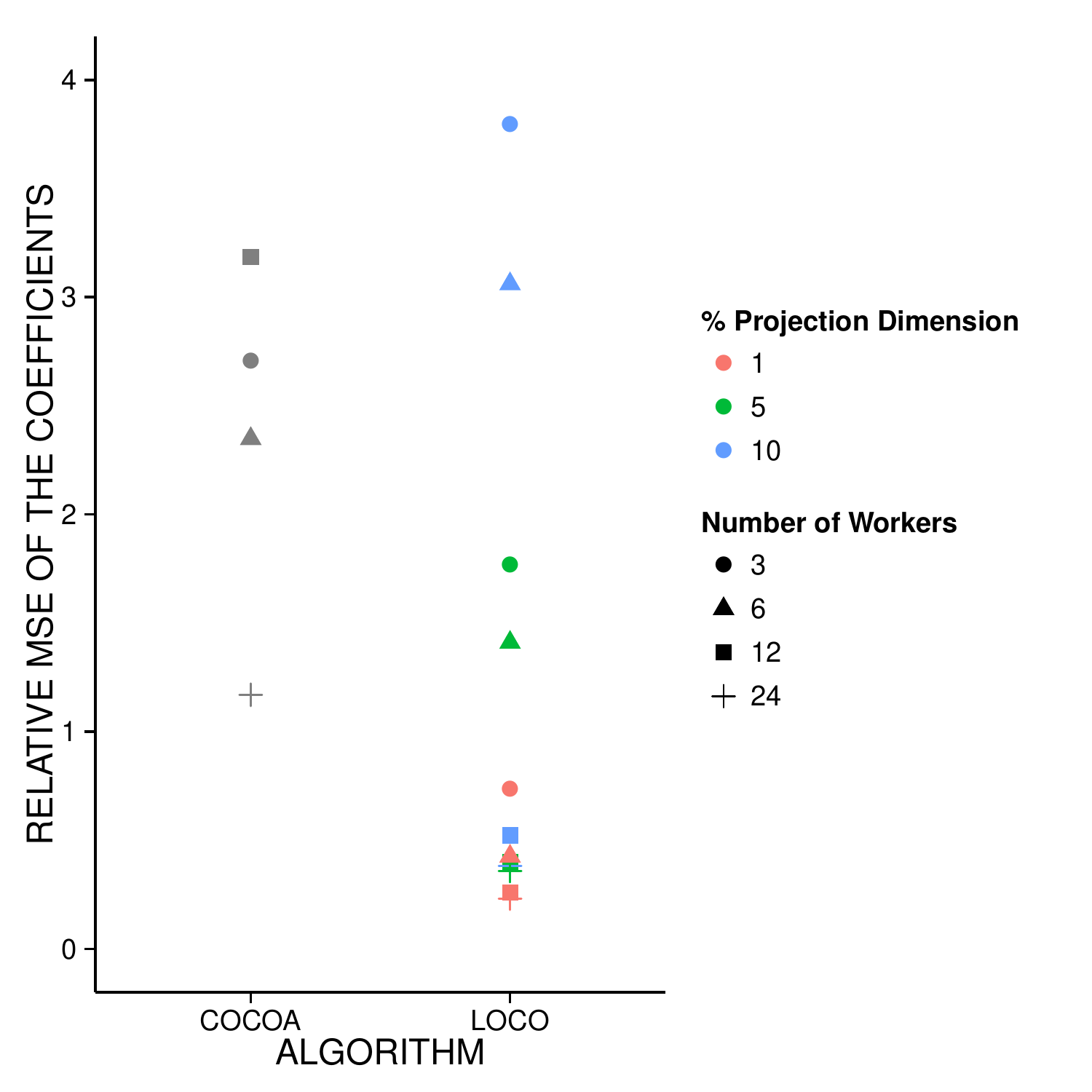}
    \label{fig:coeff_error_150K}
}
\hspace{-0.5cm}
\subfloat[\xspace]{
    \includegraphics[width=0.48\textwidth, keepaspectratio=true]{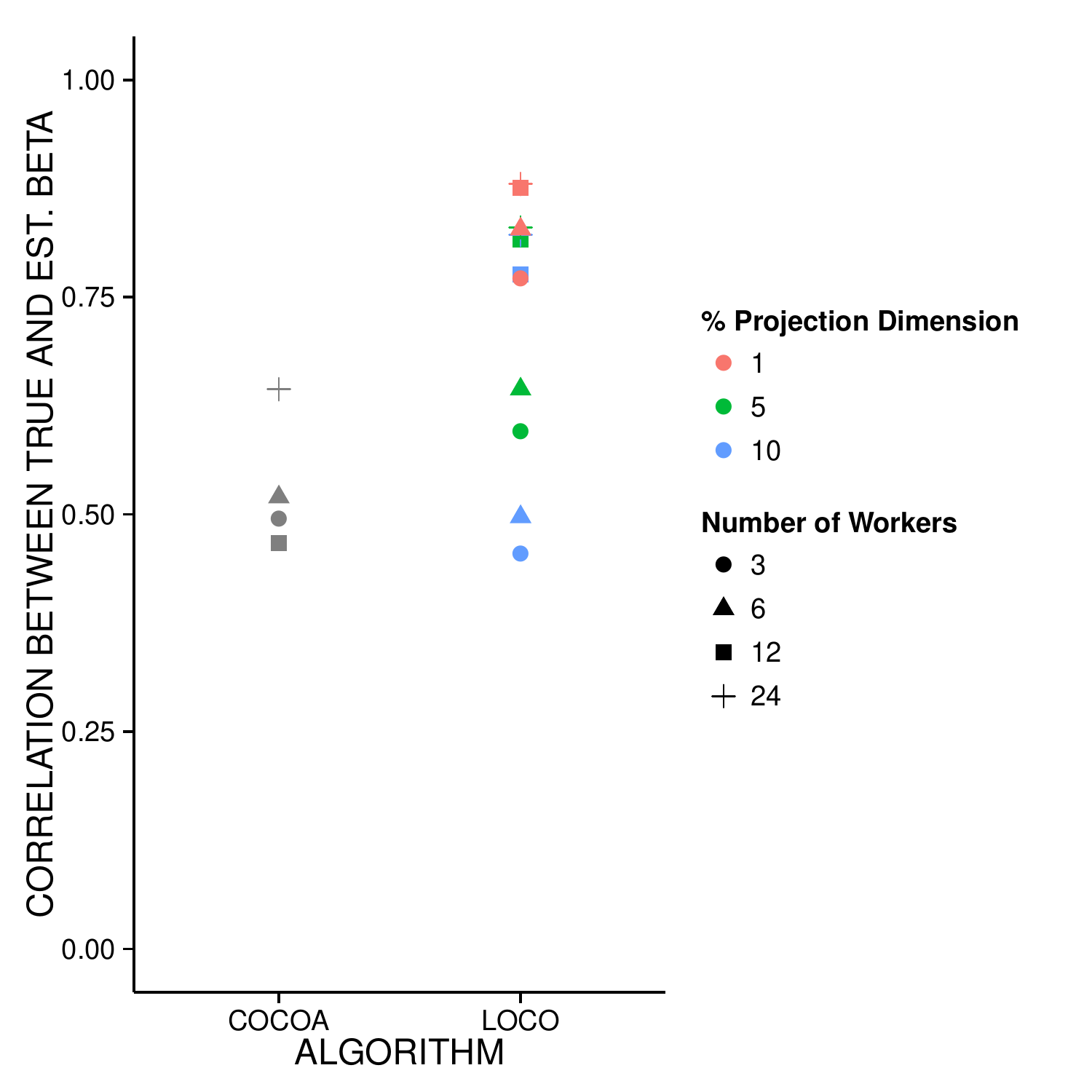}
    \label{fig:coeff_corr_150K}
}
\vspace{-5pt}
\caption{\protect\subref{fig:coeff_error_150K} Relative MSE and \protect\subref{fig:coeff_corr_150K} correlation between estimated and true coefficients when $\dims=150,000$. \label{fig:150K_coeffs}}
\end{centering}
\vspace{-0.3cm}
\end{figure*}

\paragraph{Simulated Gaussian data.}
We consider two large-scale simulated problems. The data is generated from a Gaussian distribution with mean zero and a block-wise covariance matrix such that the features are \emph{not} independent and the block structure is not known to the algorithm \emph{a priori}. Since the features in each block are correlated, this implies that the data is effectively low rank. That is, it has a number of large singular values equal to the number of blocks. The data simulation method is described in full detail in \S\ref{sec:datagen}.

\paragraph{Scenario one.} The first scenario we consider is $\samp = 4000$ and $\dims = 150,000$. This results in $600M$ non-zeros, translating into a file size of $10$GB. The test set has additional $\samp_{test}=1,000$ observations. The data has rank $r=150$, the within-block correlation is 0.7 and the signal to noise ratio is 1.

According to Theorem \ref{thm_ridge}, increasing $\dimsks$ will improve the prediction error. Since for different number of workers, $\dimsk=\dims/\blocks$ is different, the random projection dimension, $(\blocks - 1) \dimsks$ is chosen relative to $\dims - \dimsk$, i.e. $(\blocks - 1) \dimsks = \{0.01, 0.05, 0.1 \}\times (\dims - \dimsk)$. We label the corresponding results in Figure~\ref{fig:150K} with \loco 1, \loco 5 and \loco 10.

Figure~\ref{fig:150K}\subref{fig:error_150K} shows the normalized training and test MSE for $\blocks =\{3,6,12, 24\}$. As the size of the projection dimension $\dimsks$ increases, the performance of \loco improves and approaches that of \cocoa. The main difference between \cocoa and the different runs of \loco lies in the training error -- the differences between the test errors are very small. This suggests that a small projection dimension might suffice if the performance on unseen test data is of primary importance.

Figure~\ref{fig:150K}\subref{fig:speedup_150K} compares the relative speedup for increasing $\blocks$ for \loco and \cocoa , averaged over 5 trials. \cocoa exhibits near-linear speedup for up to 12 workers but as more workers are added, overall running time increases due to communication overhead and causes a relative slowdown. In contrast, \loco exhibits better-than-linear speedup between $3$ and $24$ workers as the size of the communicated matrices and the dimensionality of the local optimization problems decreases.

Figure~\ref{fig:150K_coeffs} shows the \subref{fig:coeff_error_150K} relative MSE and \subref{fig:coeff_corr_150K} correlation between the true coefficients and the coefficients returned by \loco and \cocoa. These figures show that \loco is able to obtain good estimates of the true coefficients. Contrasting these figures with Figure~\ref{fig:150K}\subref{fig:error_150K} suggest an inverse relationship between accuracy of estimating coefficients and prediction performance. This can be explained by the fact that since the data is low rank the difference between the solutions may lie in the null space of $\Xt$ which does not adversely affect the accuracy of the estimated responses.

\paragraph{Scenario two.} The second scenario we consider is $\samp=8,000$, $\dims = 500,000$ and $r=500$. Since the data is fully dense, there are 4 \emph{billion} non-zeros\footnote{Comparable in number with the experiment size of \citet{Richtarik:2013} and \citet{Peng:2013} with the key difference that we do not impose block sparsity in the data like \citet{Richtarik:2013} and we do not simply sample $\Xt$ from $\N(0,1)$ like \citet{Peng:2013}.}. Now the size of the data starts to become impractical for a single machine (training data is $\geq 64$GB) and the distributed nature of \loco is advantageous. We compare the performance of \loco against \cocoa for $K=\{50,100,200\}$.

Figure~\ref{fig:500K}\subref{fig:error_500K} shows the normalized mean-squared prediction error achieved by \loco, using $(\blocks - 1) \dimsks = \lbrace 0.01, 0.02 \rbrace \times (\dims - \dimsk)$, and \cocoa. \loco is again able to achieve good test performance, comparable to \cocoa. Figure~\ref{fig:500K}\subref{fig:speedup_500K} shows that \loco obtains a $1.5\times$ speedup when increasing from $50$ to $200$ workers whereas the run time of \cocoa increases by a factor larger than $3$, resulting in a speedup of $0.32$.

The reason for the relatively smaller speedup of \loco compared with the smaller scale experiment is that the relative reduction in local dimensionality from increasing $\blocks$ get smaller for larger $\blocks$ (assuming the projection dimension is fixed proportionally to $\dims - \dimsk$).

In scenario one, for \loco 1, the local problems are largest when $\blocks=3$. This corresponds to a local dimensionality of $\dimsk + (\blocks-1)\dimsks = 51000$. Using four times as many workers, $\blocks=12$, the local dimensionality is $13875$. This represents a decrease in local problem size of more than $72\%$ which explains the linear speedup. In scenario two, the largest local problems have size $14900$ when $\blocks=50$. When $\blocks=200$, the local dimensionality is $7475$. This represents a decrease in local dimensionality of $50\%$ for a four-fold increase in $\blocks$, explaining the smaller speedup.

Figure~\ref{fig:500K_coeffs} shows the \subref{fig:coeff_error_500K} relative MSE and \subref{fig:coeff_corr_500K} correlation between the true coefficients and the coefficients returned by each method. Both \loco and \cocoa are able to estimate coefficients which are close to the true ones.

In summary, what these results on simulated data show is that as the number of machines increases, \loco is often able to achieve significant speedup over \cocoa at the expense of a small loss of prediction accuracy. Although the differences in training error are more noticeable, these may be neglected if generalization performance is of primary interest.

\begin{figure*}[!tp]
\begin{centering}
\subfloat[\xspace]{
    \includegraphics[trim=0 32 0 0, clip, width=0.48\textwidth, keepaspectratio=true]{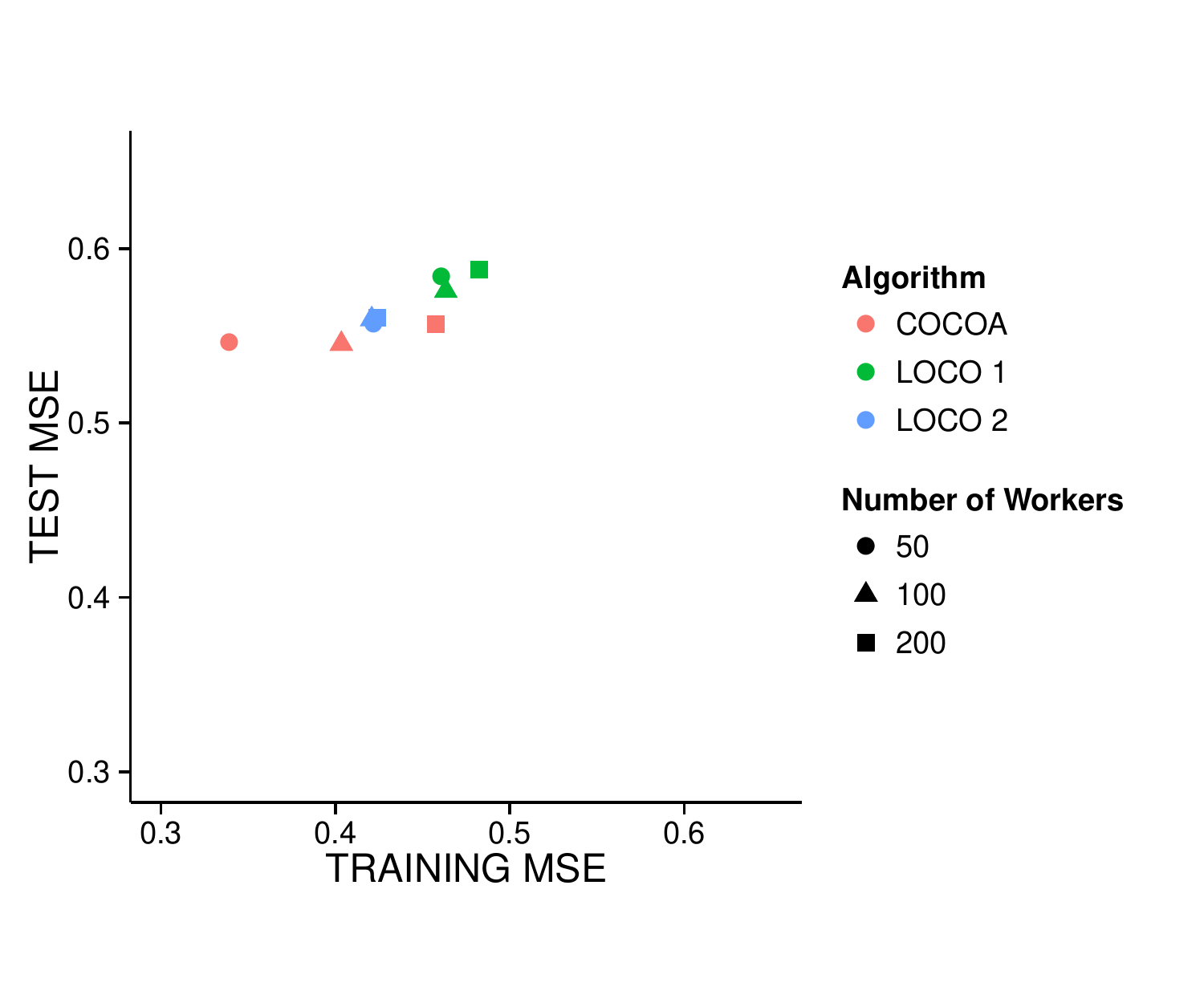}
    \label{fig:error_500K}
}
\hspace{-0.5cm}
\subfloat[\xspace]{
    \includegraphics[width=0.48\textwidth, keepaspectratio=true]{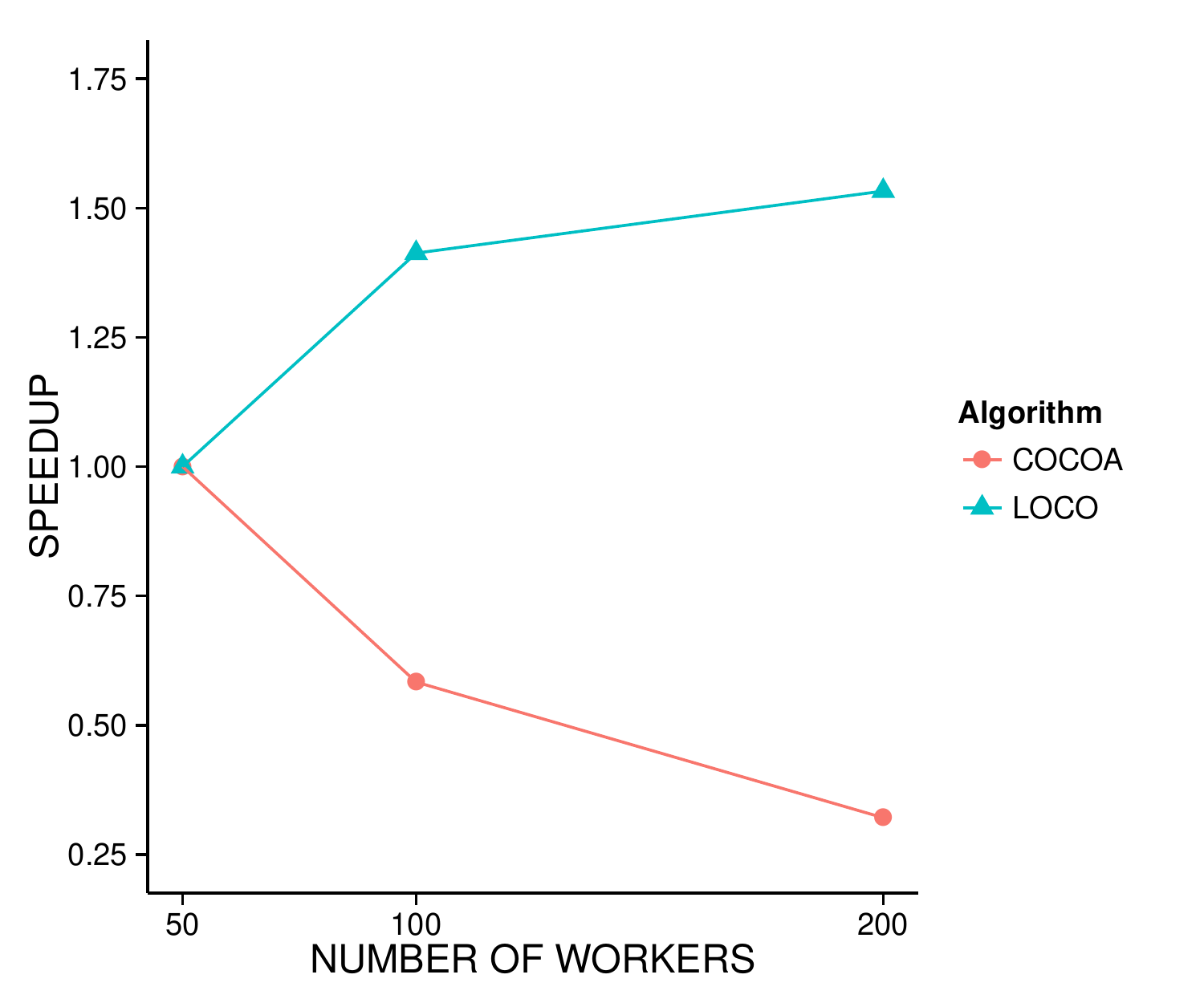}
    \label{fig:speedup_500K}
}
\vspace{-5pt}
\caption{\protect\subref{fig:error_500K} Training and test error and \protect\subref{fig:speedup_500K} relative speedup for different number of workers when $\dims=500,000$. \label{fig:500K}}
\end{centering}
\vspace{-0.3cm}
\end{figure*}

\begin{figure*}[!tp]
\begin{centering}
\subfloat[\xspace]{
    \includegraphics[width=0.48\textwidth, keepaspectratio=true]{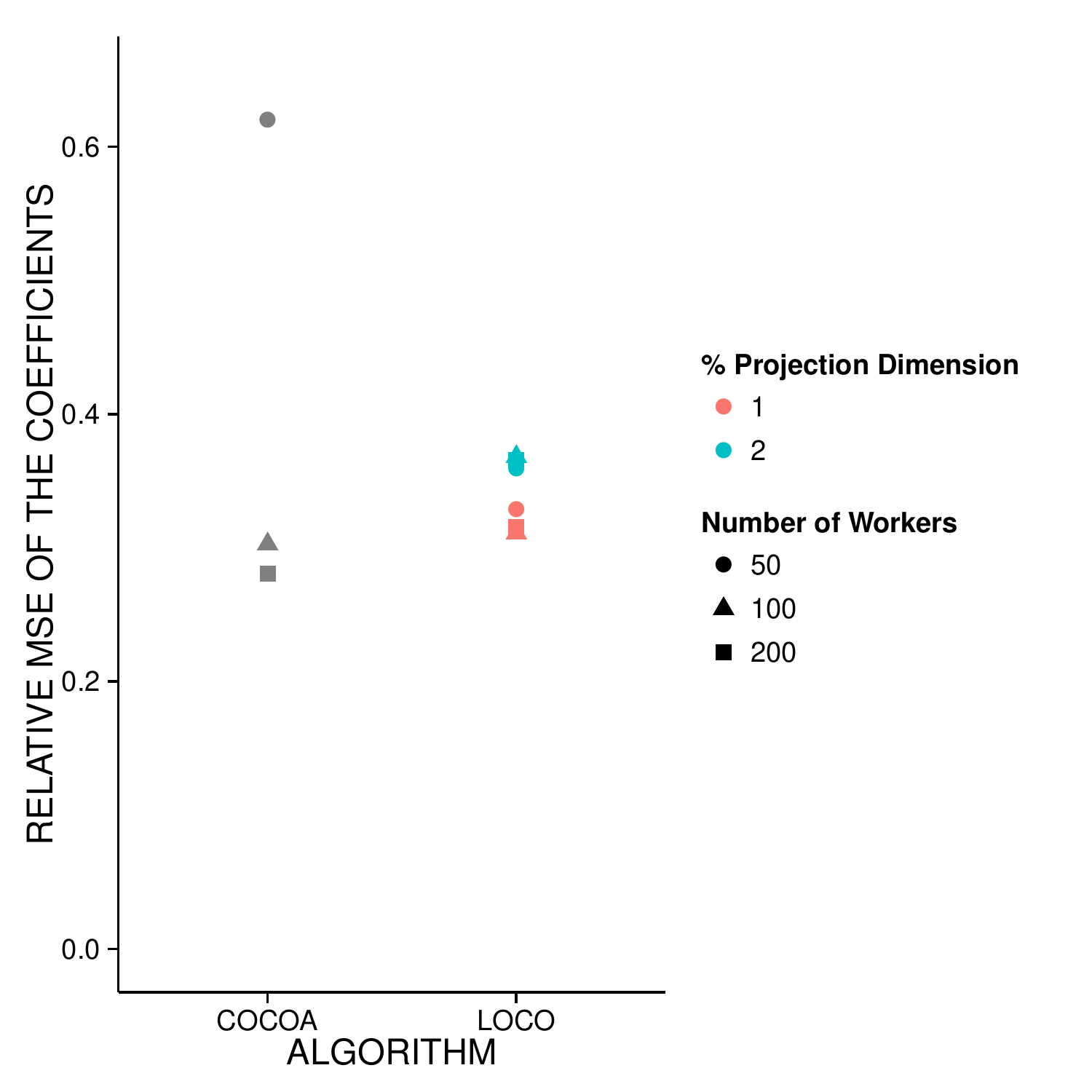}
    \label{fig:coeff_error_500K}
}
\hspace{-0.5cm}
\subfloat[\xspace]{
    \includegraphics[width=0.48\textwidth, keepaspectratio=true]{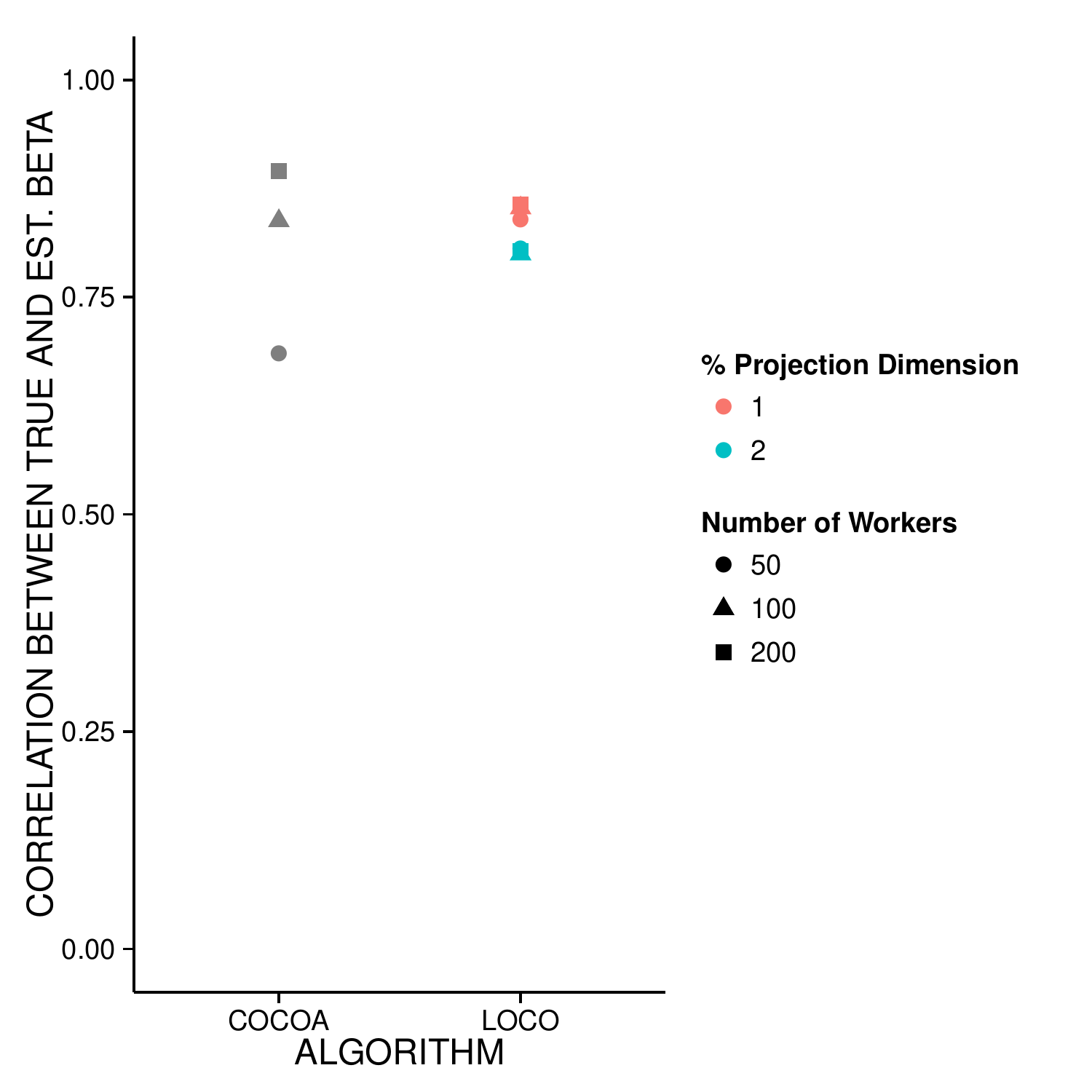}
    \label{fig:coeff_corr_500K}
}
\vspace{-10pt}
\caption{\protect\subref{fig:coeff_error_500K} Relative MSE and \protect\subref{fig:coeff_corr_500K} correlation between estimated and true coefficients when $\dims=500,000$. \label{fig:500K_coeffs}}
\end{centering}
\vspace{-0.3cm}
\end{figure*}

\paragraph{Climate data.} Finally, we present an application to a real-world problem in climate modeling.  The data we consider is part of the CMIP5 climate modeling ensemble, specifically the data are taken from control simulations of the GISS global circulation model \citep{schmidt2014configuration}. We try a simple forecast of  the global temperature  based on the temperature pattern observed a month earlier.  This allows to quantify which anomalies in the temperature pattern are persistent over time-scales of a month and which anomalies in the temperature disappear on faster time-scales.   We pick as response $\y$ here the global average temperature in February (results are very similar for other months). The $\dims = 10368$ features are the January temperatures at 10368 grid points spread across the globe.
The model simulates the climate for a range of 531 years and we use the output from two control simulation runs. The data set is split into training (80\%) and test set (20\%), resulting in $n_{\text{train}} = 849$ and $n_{\text{test}} = 213$.

In Figure~\ref{fig:beta_comparison} we compare the estimated coefficients for five methods in addition to the full solution. Three of these methods apply to the non-distributed setting and for the distributed setting we show the results of \loco and \cocoa. In Figure~\ref{fig:beta_comparison}
\begin{enumerate}
	\item[\subref{fig:beta_full}] shows the coefficients estimated in the non-distributed setting with SDCA. 
	\item[\subref{fig:beta_naive_diag}] shows the coefficients returned by the naive single-machine approximation 
	$$
	\solnest^{\text{diag}} = \text{diag}(\Xt\tr\Xt)^{-1}\Xt\tr\y ,
	$$ 
	which is equivalent to assuming independence between the features.
	\item[\subref{fig:beta_proj}] shows the coefficients that are returned when the dimensionality of the design matrix is first compressed with a random projection prior to estimating the coefficients using SDCA in this low-dimensional space and then projected back to the original space.
	\item[\subref{fig:beta_rowcompression}] shows the coefficients returned as a result of compressing the rows of $\Xt$ with a random projection to $\subsamp = \samp/2$ prior to performing ridge regression.
	\item[\subref{fig:beta_loco}] shows the coefficients returned by \loco, distributed over 4 workers, compressing each worker's raw features ($\dimsk = 2592$) to $10\%$ of the dimensionality, i.e.  $\dimsks = 260$ and concatenating these representations.
	\item[\subref{fig:beta_cocoa}] shows the coefficients returned by CoCoA.
\end{enumerate}

The coefficients returned by \loco are similar to the optimal non-distributed solution. This behaviour is expected as a consequence of Theorem \ref{thm_ridge}. On the other hand, the up-projected coefficients are a poor approximation to the optimal solution which justifies our distributed approach to ridge regression over standard dimensionality reduction approaches. 
The coefficients returned by \cocoa are also similar to the optimal non-distributed solution. 

The diagonal approximation obtains a large $MSE$ which is expected due to the important correlations between the features which are neglected in this approach. Lastly, due to the large ratio between $\dims$ and $\samp$, for the row-compression approach both the $MSE$ and the approximation quality of the coefficients suffer due to reducing the effective sample size.

In this application, the regression coefficients have a clear physical interpretation. The regression coefficients in panel (a) of Figure~\ref{fig:beta_comparison} show a near-optimal ridge regression solution computed on a single machine. The regression coefficients are essentially 0 across the oceans, showing that any deviation of sea surface temperatures is not relevant for persistent global temperature anomalies in the winter months of the northern hemisphere. The main contribution stems from large regression coefficients over the landmasses of the northern hemisphere which show a large variability of temperature in these winter months due to the possible influx of cooler arctic air and the regression shows that these anomalies are persistent on a monthly time-scale and allow to forecast the global temperature anomaly a months later in February. 
    Ensuring the estimated coefficients returned by \loco are close to the optimal coefficients is in applications like this at least as important as obtaining a low prediction error.

\begin{figure*}[!tp]
\begin{centering}

\subfloat[Single machine: Full solution ($MSE = 0.5756$)]{
    \includegraphics[trim=58 73 0 50, clip, width=0.45\textwidth, keepaspectratio=true]{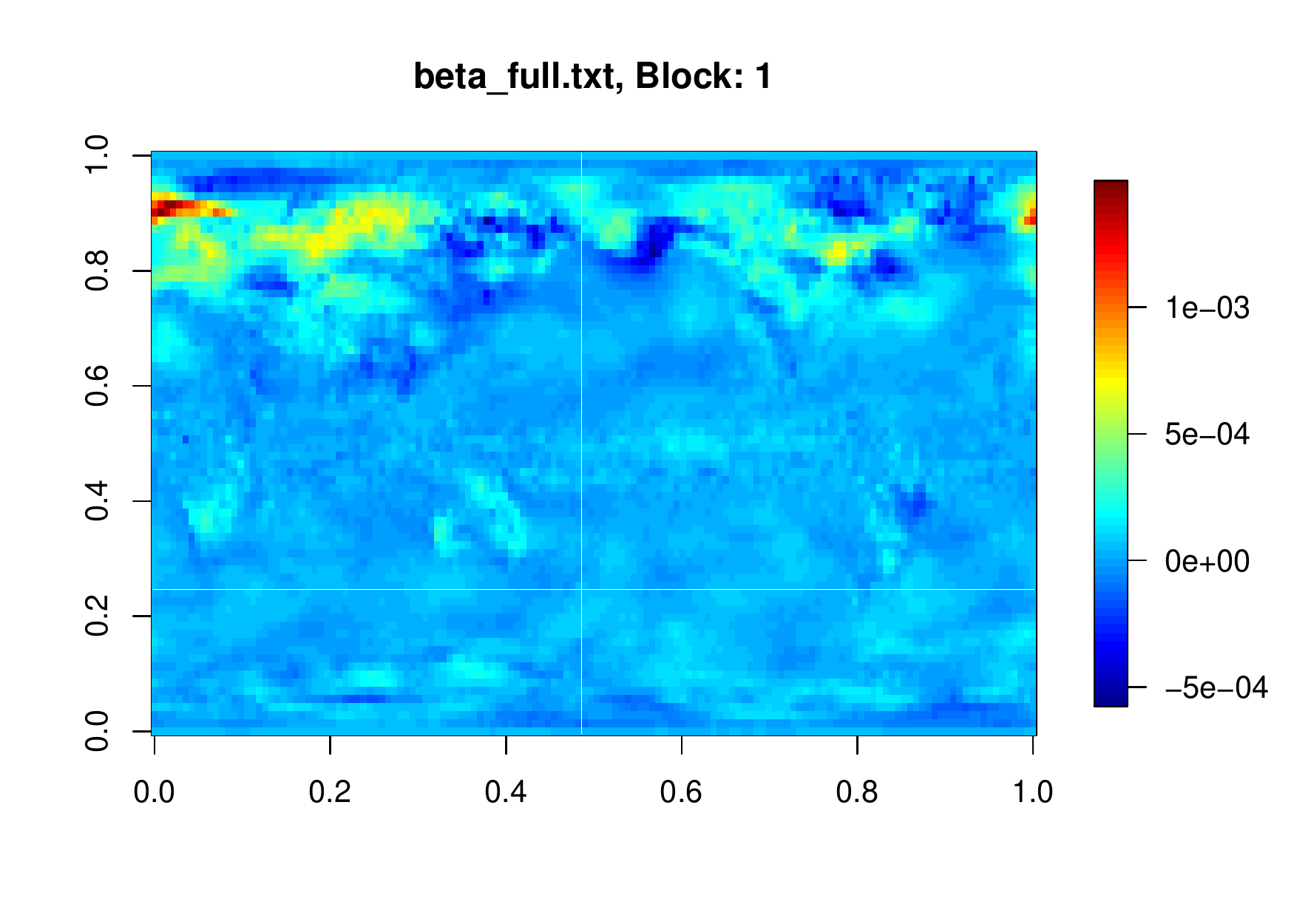}
    \label{fig:beta_full}
}
\hspace{0.5cm}
\subfloat[Single machine: Diagonal approximation ($MSE = 53956.4377$)]{
    \includegraphics[trim=58 73 0 50, clip, width=.45\textwidth, keepaspectratio=true]{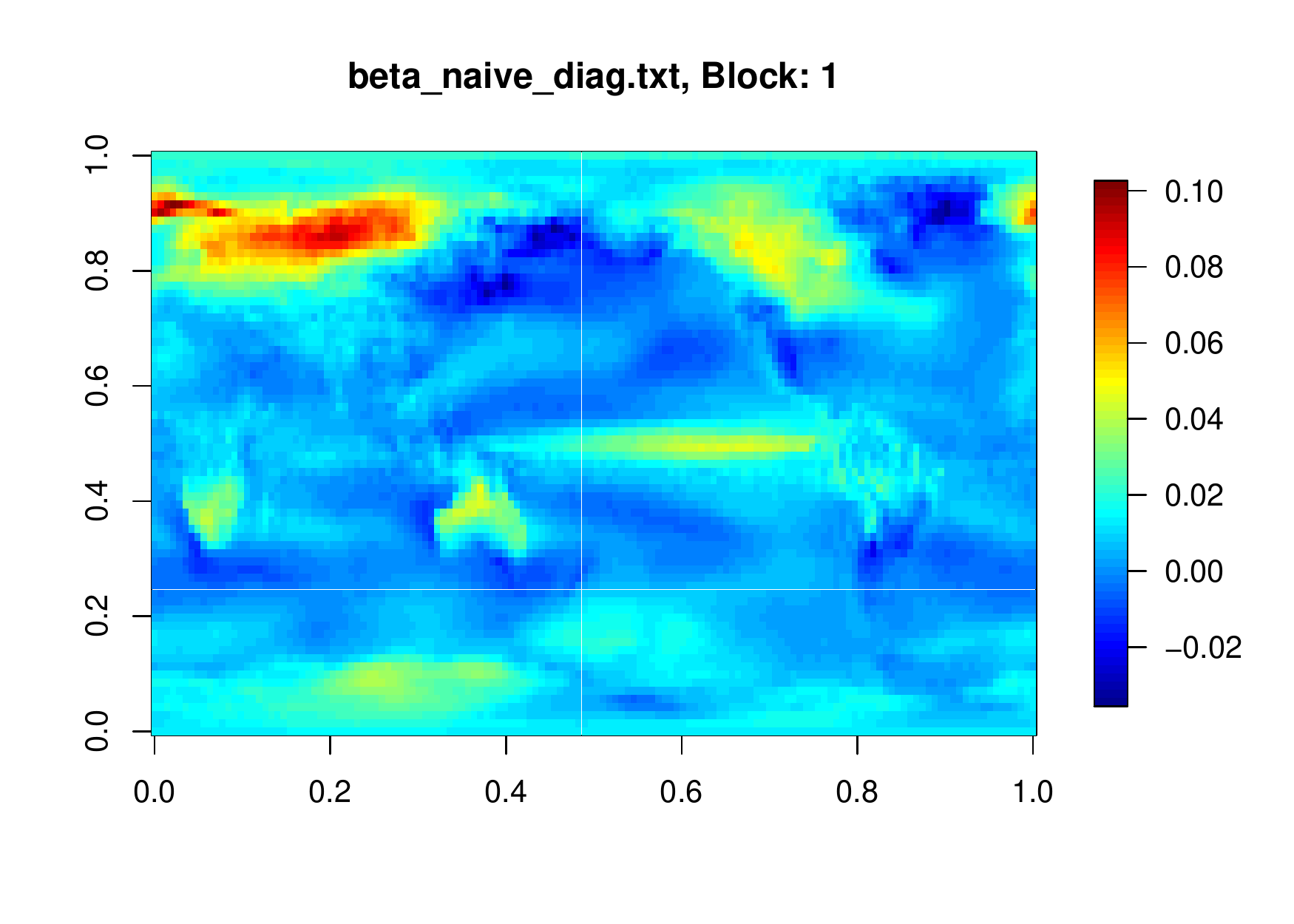}
    \label{fig:beta_naive_diag}
}

\subfloat[Single machine: Column-wise compression ($MSE = 0.5661$)]{
    \includegraphics[trim=58 73 0 50, clip, width=0.45\textwidth, keepaspectratio=true]{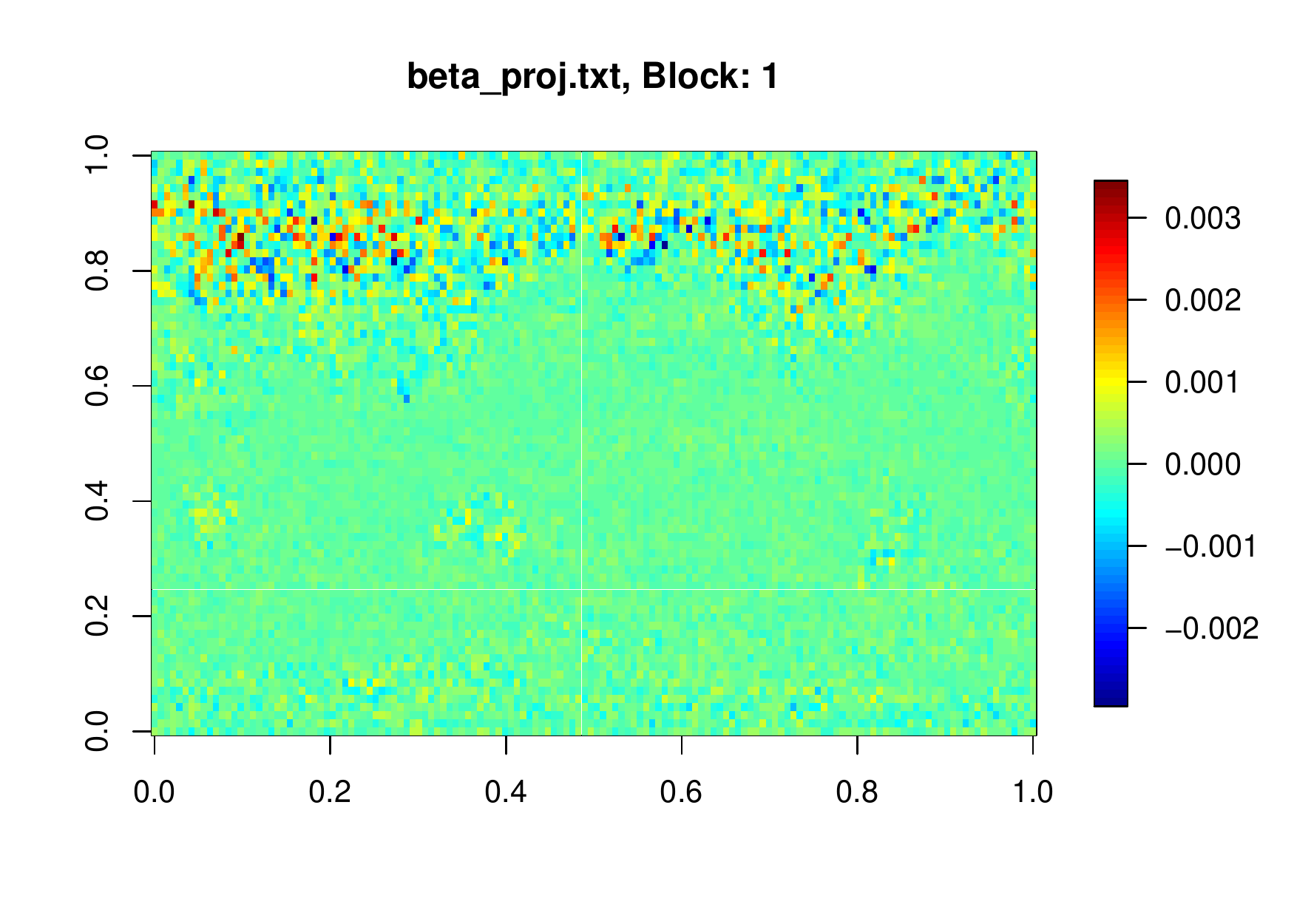}
    \label{fig:beta_proj}
}
\hspace{0.5cm}
\subfloat[Single machine: Row-wise compression ($MSE = 0.9977$)]{
    \includegraphics[trim=58 73 0 50, clip, width=0.45\textwidth, keepaspectratio=true]{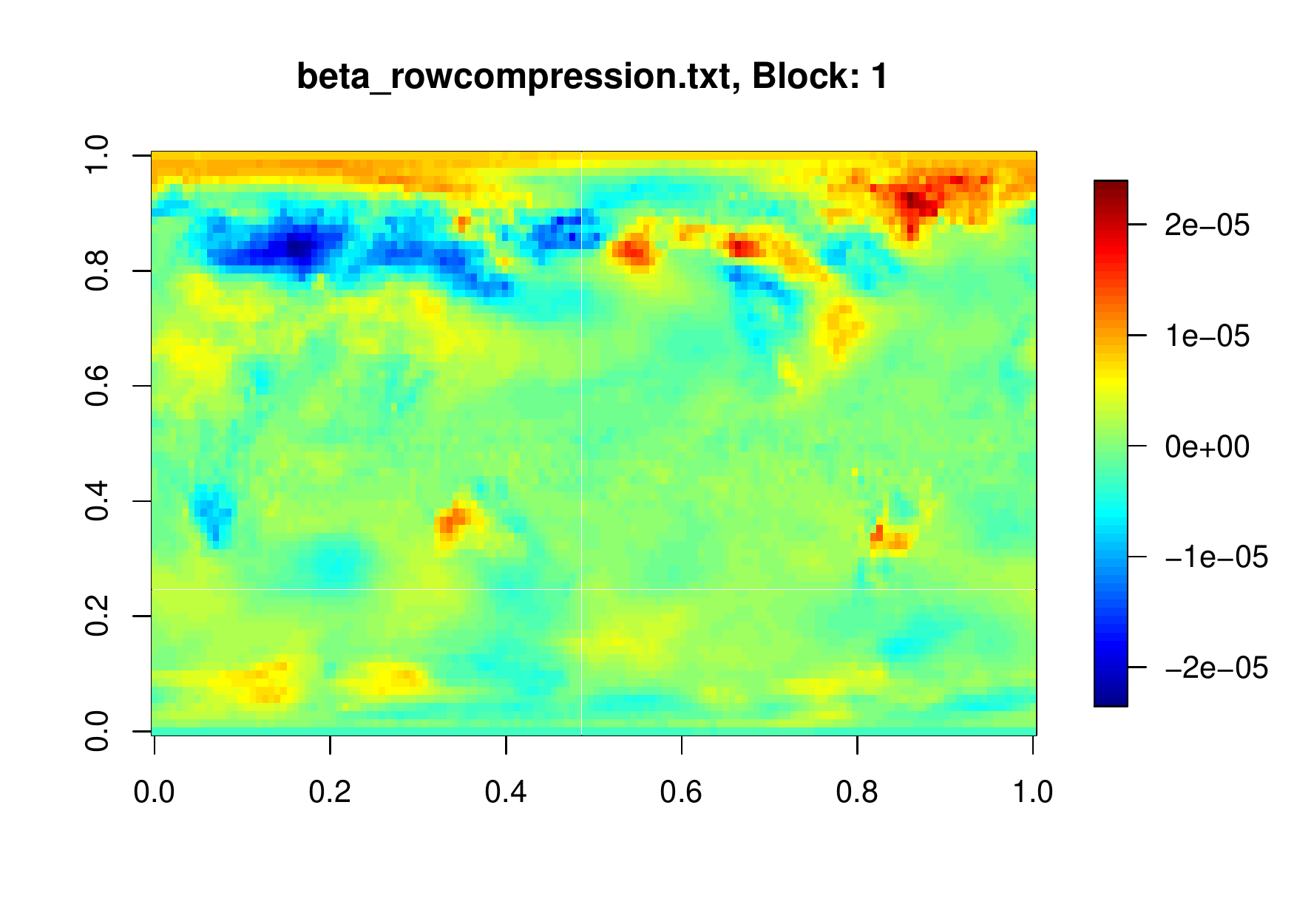}
    \label{fig:beta_rowcompression}
}

\subfloat[Distributed setting: \loco  ($MSE = 0.5782$)]{
    \includegraphics[trim=58 73 0 50, clip, width=0.45\textwidth, keepaspectratio=true]{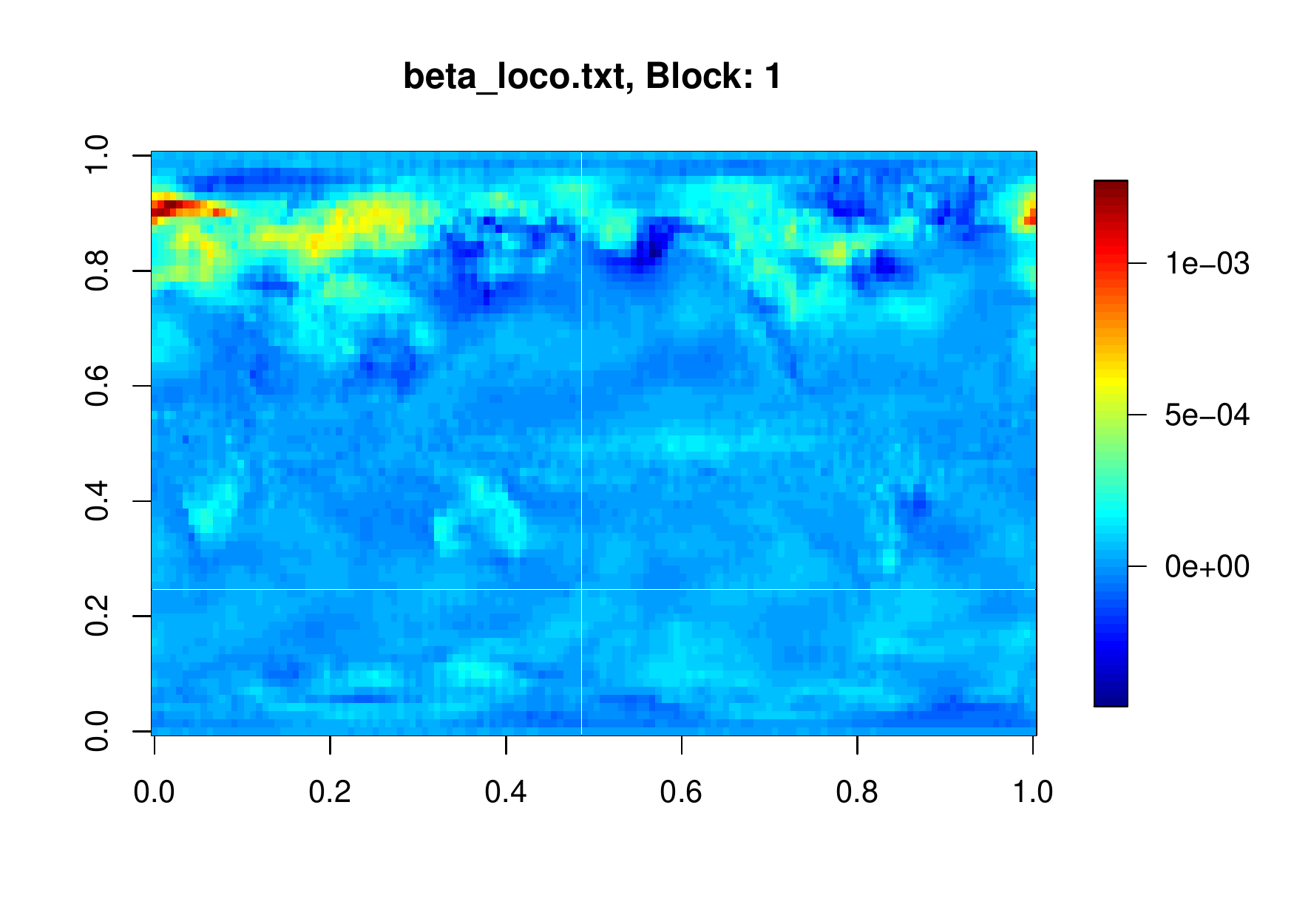}
    \label{fig:beta_loco}
}
\hspace{0.5cm}
\subfloat[Distributed setting: \cocoa  ($MSE = 0.5795$)]{
    \includegraphics[trim=58 73 0 50, clip, width=0.45\textwidth, keepaspectratio=true]{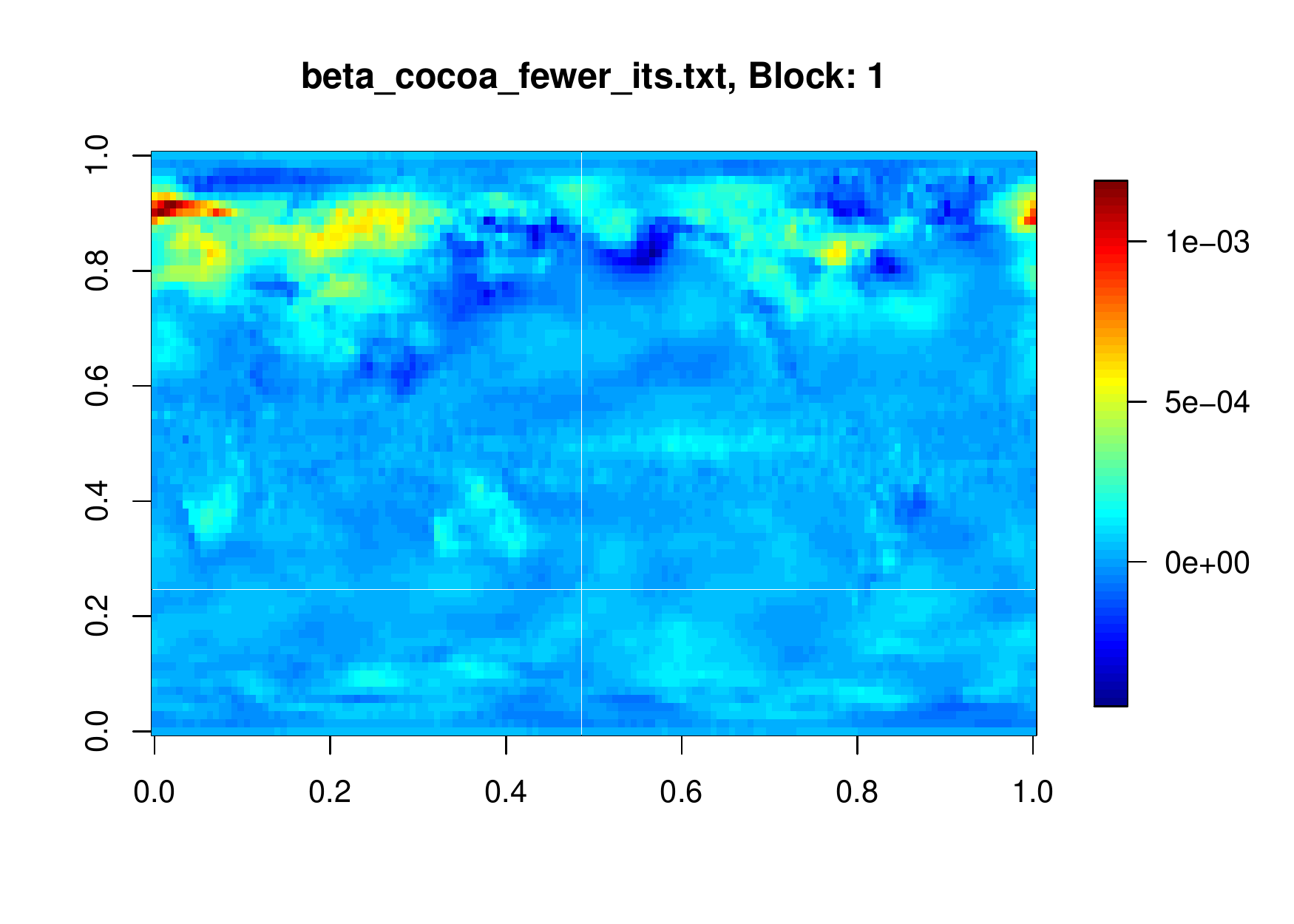}
    \label{fig:beta_cocoa}
}

\caption{Climate data: Comparison of different methods which return coefficients lying in the original space. The regression coefficients are shown as maps with the prime median (passing through London) corresponding to the left and right edge of the plot. The Pacific Ocean is occupying the centre of each map. \label{fig:beta_comparison}}
\end{centering}
\vspace{-0.3cm}
\end{figure*}

\section{Discussion}
\vspace{-0.2cm}
In this work we have presented \loco, a simple algorithm for distributed ridge regression -- requiring minimal communication and no synchronization -- based on random projections. We have shown theoretically and empirically that \loco achieves small additional error compared with the optimal ridge regression solution. It obtains significant speedups with the number of workers without making any additional assumptions about sparsity in the data. If the data is very sparse, we expect to see additional performance gains with a sparse random projection. 

	\loco is useful in settings where physical meaning can be assigned to estimated coefficients and it is therefore important to estimate coefficients in the original data space. In such cases, as illustrated in the experiments presented on climate model data, \loco is able to preserve structure in the estimated coefficients which is lost when performing standard dimensionality reduction. Although currently our results are specific for ridge regression, we expect that the same principles can be generalized to a larger class of estimation problems. 

As mentioned in the introduction, distributed optimization -- where no single worker sees all of the data -- is a natural paradigm when preserving privacy is required. Additionally, the class of J-L projections that we use have been shown to preserve differential privacy \citep{Blocki:2012}. We aim to explore the connection between \loco and privacy aware learning.

Finally, \citet{zhang:2013inf} have recently established bounds on the minimum amount of communication necessary for a distributed estimation task to achieve minimax optimal risk. It would be interesting investigate how \loco fits into this framework since the distribution strategy of \loco (across features rather than rows) differs from most commonly analysed methods.



\acks{We acknowledge the World Climate Research Programme's Working
  Group on Coupled Modelling, which is responsible for CMIP, and we
  thank Reto Knutti and Jan Sedlacek from the Climate Physics group at ETH Zurich
  for producing and making available their model output and for their kind help with 
  the preparation of the data. For CMIP the
  U.S. Department of Energy's Program for Climate Model Diagnosis and
  Intercomparison provides coordinating support and led development of
  software infrastructure in partnership with the Global Organization
  for Earth System Science Portals. 

We would also like to thank Martin Jaggi for valuable discussions on optimization, Stefan
Deml for contributing to the implementation of the software and Rok Roskar for help with Spark. }


\bibliographystyle{plain}
\bibliographystyle{abbrv}
\bibliography{loco_jmlr}

\begin{thebibliography}{28}
\providecommand{\natexlab}[1]{#1}
\providecommand{\url}[1]{\texttt{#1}}
\expandafter\ifx\csname urlstyle\endcsname\relax
  \providecommand{\doi}[1]{doi: #1}\else
  \providecommand{\doi}{doi: \begingroup \urlstyle{rm}\Url}\fi

\bibitem[Achlioptas(2003)]{achlioptas2003}
Dimitris Achlioptas.
\newblock {Database-friendly random projections: Johnson-Lindenstrauss with
  binary coins}.
\newblock \emph{Journal of Computer and System Sciences}, 2003.

\bibitem[Agarwal and Duchi(2011)]{agarwal:2011}
Alekh Agarwal and John~C. Duchi.
\newblock Distributed delayed stochastic optimization.
\newblock In \emph{NIPS}, pages 873--881, 2011.

\bibitem[Ailon and Chazelle(2009)]{Ailon:2009}
Nir Ailon and Bernard Chazelle.
\newblock The fast johnson-lindenstrauss transform and approximate nearest
  neighbors.
\newblock \emph{SIAM Journal on Computing}, 39\penalty0 (1):\penalty0 302--322,
  2009.

\bibitem[Bach(2012)]{Bach:2012vm}
Francis Bach.
\newblock {Sharp analysis of low-rank kernel matrix approximations}.
\newblock \emph{arXiv preprint arXiv:1208.2015}, 2012.

\bibitem[Blocki et~al.(2012)Blocki, Blum, Datta, and Sheffet]{Blocki:2012}
Jeremiah Blocki, Avrim Blum, Anupam Datta, and Or~Sheffet.
\newblock The johnson-lindenstrauss transform itself preserves differential
  privacy.
\newblock In \emph{Foundations of Computer Science (FOCS), 2012 IEEE 53rd
  Annual Symposium on}, pages 410--419. IEEE, 2012.

\bibitem[Boutsidis and Gittens(2012)]{Boutsidis:2012tv}
Christos Boutsidis and Alex Gittens.
\newblock {Improved matrix algorithms via the Subsampled Randomized Hadamard
  Transform}.
\newblock 2012.
\newblock arXiv:1204.0062v4 [cs.DS].

\bibitem[Bradley et~al.(2011)Bradley, Kyrola, Bickson, and
  Guestrin]{Bradley:2011}
Joseph~K. Bradley, Aapo Kyrola, Danny Bickson, and Carlos Guestrin.
\newblock Parallel coordinate descent for $l_1$-regularized loss minimization.
\newblock In \emph{International Conference on Machine Learning}, 2011.

\bibitem[Drineas et~al.(2011)Drineas, Magdon-Ismail, Mahoney, and
  Woodruff]{Drineas:2011ts}
Petros Drineas, Malik Magdon-Ismail, Michael~W. Mahoney, and David~P. Woodruff.
\newblock {Fast approximation of matrix coherence and statistical leverage}.
\newblock September 2011.
\newblock arXiv:1109.3843v2 [cs.DS].

\bibitem[Duchi et~al.(2013)Duchi, Jordan, and McMahan]{Duchi:2013}
John~C. Duchi, Michael~I. Jordan, and H.~Brendan McMahan.
\newblock Estimation, optimization, and parallelism when data is sparse.
\newblock In \emph{Advances in Neural Information Processing Systems}, 2013.

\bibitem[Halko et~al.(2011)Halko, Martinsson, and Tropp]{Halko:2011kg}
Nathan Halko, Per-Gunnar Martinsson, and Joel~A. Tropp.
\newblock {Finding structure with randomness: Probabilistic algorithms for
  constructing approximate matrix decompositions}.
\newblock \emph{SIAM Review}, 53\penalty0 (2):\penalty0 217--288, 2011.

\bibitem[Hastie et~al.(2009)Hastie, Tibshirani, and Friedman]{esl}
Trevor Hastie, Robert Tibshirani, and Jerome Friedman.
\newblock \emph{The Elements of Statistical Learning}.
\newblock Springer Series in Statistics. Springer New York Inc., New York, NY,
  USA, 2009.

\bibitem[Jaggi et~al.(2014)Jaggi, Smith, Tak{\'a}c, Terhorst, Krishnan,
  Hofmann, and Jordan]{jaggi2014communication}
Martin Jaggi, Virginia Smith, Martin Tak{\'a}c, Jonathan Terhorst, Sanjay
  Krishnan, Thomas Hofmann, and Michael~I. Jordan.
\newblock Communication-efficient distributed dual coordinate ascent.
\newblock In \emph{Advances in Neural Information Processing Systems}, pages
  3068--3076, 2014.

\bibitem[Kab{\'a}n(2014)]{kaban:2014}
Ata Kab{\'a}n.
\newblock New bounds on compressive linear least squares regression.
\newblock In \emph{Artificial Intelligence and Statistics}, 2014.

\bibitem[Le et~al.(2013)Le, Sarlos, and Smola]{Le:2013}
Quoc Le, Tamas Sarlos, and Alex Smola.
\newblock Fastfood — approximating kernel expansions in loglinear time.
\newblock In \emph{ICML}, 2013.

\bibitem[Liu and Ihler(2014)]{liu2014distributed}
Qiang Liu and Alex~T. Ihler.
\newblock Distributed estimation, information loss and exponential families.
\newblock In \emph{Advances in Neural Information Processing Systems}, pages
  1098--1106, 2014.

\bibitem[Lu et~al.(2013)Lu, Dhillon, Foster, and Ungar]{lu:2013}
Yichao Lu, Paramveer Dhillon, Dean~P. Foster, and Lyle Ungar.
\newblock Faster ridge regression via the subsampled randomized hadamard
  transform.
\newblock In \emph{Advances in Neural Information Processing Systems 26}, pages
  369--377, 2013.

\bibitem[Mahoney(2011)]{Mahoney:2011te}
Michael~W. Mahoney.
\newblock {Randomized algorithms for matrices and data}.
\newblock April 2011.
\newblock arXiv:1104.5557v3 [cs.DS].

\bibitem[McWilliams et~al.(2014)McWilliams, Krummenacher, Lucic, and
  Buhmann]{mcwilliams2014fast}
Brian McWilliams, Gabriel Krummenacher, Mario Lucic, and Joachim~M. Buhmann.
\newblock Fast and robust least squares estimation in corrupted linear models.
\newblock In \emph{Advances in Neural Information Processing Systems}, pages
  415--423, 2014.

\bibitem[Niu et~al.(2011)Niu, Recht, R\'e, and Wright]{hogwild}
Feng Niu, Benjamin Recht, Christopher R\'e, and Stephen~J. Wright.
\newblock Hogwild!: A lock-free approach to parallelizing stochastic gradient
  descent.
\newblock In \emph{NIPS}, 2011.

\bibitem[Peng et~al.(2013)Peng, Yan, and Yin]{Peng:2013}
Zhimin Peng, Ming Yan, and Wotao Yin.
\newblock Parallel and distributed sparse optimization.
\newblock In \emph{Preprint}, 2013.

\bibitem[Richtarik and Takac(2013)]{Richtarik:2013}
Peter Richtarik and Martin Takac.
\newblock Distributed coordinate descent method for learning with big data.
\newblock In \emph{Preprint}, 2013.

\bibitem[Schmidt et~al.(2014)Schmidt, Kelley, Nazarenko, Ruedy, Russell,
  Aleinov, Bauer, Bauer, Bhat, Bleck, et~al.]{schmidt2014configuration}
Gavin~A. Schmidt, Max Kelley, Larissa Nazarenko, Reto Ruedy, Gary~L. Russell,
  Igor Aleinov, Mike Bauer, Susanne~E. Bauer, Maharaj~K. Bhat, Rainer Bleck,
  et~al.
\newblock Configuration and assessment of the {GISS ModelE2} contributions to
  the {CMIP5} archive.
\newblock \emph{Journal of Advances in Modeling Earth Systems}, 6\penalty0
  (1):\penalty0 141--184, 2014.

\bibitem[Shalev-Shwartz and Zhang(2013)]{shalev:2013}
Shai Shalev-Shwartz and Tong Zhang.
\newblock Stochastic dual coordinate ascent methods for regularized loss.
\newblock \emph{The Journal of Machine Learning Research}, 14\penalty0
  (1):\penalty0 567--599, 2013.

\bibitem[Tropp(2010{\natexlab{a}})]{Tropp:2010uo}
Joel~A. Tropp.
\newblock {Improved analysis of the subsampled randomized Hadamard transform}.
\newblock November 2010{\natexlab{a}}.
\newblock arXiv:1011.1595v4 [math.NA].

\bibitem[Tropp(2010{\natexlab{b}})]{Tropp:2010vm}
Joel~A. Tropp.
\newblock {User-friendly tail bounds for sums of random matrices}.
\newblock April 2010{\natexlab{b}}.
\newblock arXiv:1004.4389v7 [math.PR].

\bibitem[Zhang et~al.(2013{\natexlab{a}})Zhang, Duchi, Jordan, and
  Wainwright]{zhang:2013inf}
Yuchen Zhang, John Duchi, Michael Jordan, and Martin~J. Wainwright.
\newblock Information-theoretic lower bounds for distributed statistical
  estimation with communication constraints.
\newblock In \emph{Advances in Neural Information Processing Systems}, pages
  2328--2336, 2013{\natexlab{a}}.

\bibitem[Zhang et~al.(2013{\natexlab{b}})Zhang, Duchi, and
  Wainwright]{zhang:2013}
Yuchen Zhang, John~C. Duchi, and Martin~J. Wainwright.
\newblock Divide and conquer kernel ridge regression: A distributed algorithm
  with minimax optimal rates.
\newblock \emph{arXiv preprint arXiv:1305.5029}, 2013{\natexlab{b}}.

\bibitem[Zinkevich et~al.(2010)Zinkevich, Weimer, Smola, and
  Li]{zinkevich2010parallelized}
Martin Zinkevich, Markus Weimer, Alexander~J. Smola, and Lihong Li.
\newblock Parallelized stochastic gradient descent.
\newblock In \emph{NIPS}, volume~4, page~4, 2010.

\end{thebibliography}
\newpage
\appendix
\noindent{\Large{\textbf{Supplementary Information for {\sc Loco}: Distributing Ridge Regression with Random Projections} }}
\newcounter{si-sec}
\renewcommand{\thesection}{SI.\arabic{si-sec}}


Here we collect supplementary technical details, empirical results and discussion which support the results presented in the main text.

\addtocounter{si-sec}{1}

\section{Data generation} \label{sec:datagen}
Typically parallel optimization methods are evaluated on extremely sparse datasets \citep{hogwild,Duchi:2013} or uncorrelated simulated data \citep{Richtarik:2013,agarwal:2011,Peng:2013} which fulfils the types of assumptions on sparsity or low-correlations between features necessary to obtain theoretical results. Since we do not make these assumptions, we aim to show that \loco is robust to correlations between features which can be accounted for using random projections.

We generate data which have a blockwise correlation structure. Within each of the $r=1,\ldots,R$ blocks, the correlation between variables is given by $\sigma_r$. We can do this in an efficient way by constructing a symmetric matrix $\boldSigma_r \in\R^{(\dims/R) \times (\dims/R)}$ with diagonal elements $1$ and off-diagonal elements $\sigma_r$. This construction allows us to specify a different covariance matrix within each block.
We decompose $\boldSigma_r =\Lt_r\Lt_r\tr$ using the Cholesky decomposition where $\Lt_r$ is an upper triangular matrix. 

Now, the data is generated in the following way: 
\begin{enumerate}
\item	First, a $\dims$-dimensional standard Gaussian vectors is sampled according to
$$
\z\in\R^{\dims}\sim\N(0,\It_{\dims}).
$$
\item Now, we construct $\x\in\R^{\dims} = \z \sq{\Lt_1,\ldots,\Lt_R}$ so that 
$$
\bE\sq{\x\x\tr} =
\left[
\begin{array}{ccc}
\boldSigma_1 &  & \mathbf{0} \\
 & \ddots &  \\	
\mathbf{0} &  & \boldSigma_r
\end{array} 
\right]
.
$$

This ensures that the features are correlated within each block but uncorrelated between blocks. Although the resulting matrix, $\Xt\in\R^{\samp\times\dims}$ whose rows consist of samples $\x_i\tr$, $i=1,\ldots,\samp$ is nominally full rank, the size of the gap between the $R$ and $R+1$ singular values depend on the chosen values of $\sigma_r$ and $R$.

\item We sample the true regression vector in each block according to $\soln_r^{*}\in\R^{\dims/R}\sim\N(\mu_r,0.5\It_{\dims/R})$ where each $\mu_r$ is sampled uniformly at random from the integers $\{-10,\ldots,-1,1,\ldots, 10 \}$ without replacement. We construct the full coefficient vector by concatenating $\soln^* = \sq{\soln^*_1,\ldots,\soln^*_R}$.   

\item We generate the vector of responses as $\bar{\y}\in\R^\samp = \Xt \soln^*$. In order to control the signal to noise ratio we compute the size of the signal as $\sigma_s = \text{std}(\bar{\y})$ and add Gaussian noise to each response variable with variance proportional to $\sigma_s$, i.e.\ for a SNR of unity we set $\y = \bar{\y} + \sigma_s\epsilon$, where each element of $\epsilon\in\R^{\samp}$ is sampled i.i.d. from a standard Gaussian. 

\item Finally, we permute the columns of $\Xt$ and the entries in $\soln^*$ using the same random permutation so that the indices are no longer ordered according to block membership. 

\end{enumerate}

Using this routine, we can quickly generate very high dimensional and dense vectors $\x$ which have an interesting, effectively low-dimensional structure and dependencies between features.

\addtocounter{si-sec}{1}
\section{Proofs of main results}

\paragraph{Proof roadmap.} Since the modified ridge regression problem \eqref{eq:ridge_k} that each worker solves is convex, each worker will obtain a global minimizer to its own problem. In order to ensure a good solution to the global problem, we quantify the approximation error each worker incurs for the raw features with respect to the global solution since only these estimates are ultimately used in the solution $\solnest$ \loco returns. This is achieved by first bounding the difference between the risk of ridge regression in the original space $\risk(\Xt \solnrr)$ and the risk of ridge regression in the compressed space of worker $k$ $\risk(\feats_k \solnRP_k)$. Using this bound and the convexity of $f$, we will derive the final bound for the expected difference between the estimates.

Before presenting the proof of Theorem \ref{thm_ridge} we shall introduce a few necessary lemmata. Subsequently, we derive a similar bound for OLS regression.  An illustration of the bound in Theorem \ref{thm_ridge} can be found in Remark \ref{rem:discussion_bound} in \S \ref{sec:supp-rr}.

\subsection{Ridge Regression} \label{sec:supp-rr}
Recall that $\solnrr$ denotes the solution which results from solving the ridge regression problem using the original data and let $\solnrr_k = \solnRP_k $ be the solution to solving ridge regression in the compressed domain of worker $k$. 
Each of these parameter vectors can be partitioned into two parts - one containing the components corresponding to the raw features of interest while the other part contains either the remaining raw features not in $\dimset_k$ or the random features. This is clarified in the following definition. 
\begin{defn} \label{def:betas}
Assuming without loss of generality that the problem is permuted so that the raw features of worker $k$'s problem are the first $\dimsk$ columns of $\At$ and $\feats_k$ we have 
\begin{enumerate}[(a.)]
\item the ridge estimate in the original space,
$\solnrr = \sq{\solnrr_\text{raw}; \solnrr_{\minusk}},$ and
\item the ridge estimate returned by worker $k$, 
$\solnrr_k = \sq{\solnrr_{k,\text{raw}}; \solnrr_{k, \text{rp}}}.$
\end{enumerate} 
\end{defn} 
In the following we will show that $\solnrr_{k,\text{raw}}$ is approximately equal to the corresponding coordinates of the coefficient vector, \emph{(a.)} in the original space, i.e.
$$
\solnrr_{k,\text{raw}} \approx \solnrr_\text{raw}.
$$
For each worker we need to bound difference between these components as \loco eventually constructs its estimate $\solnest$ by concatenating the estimates for the raw features of each worker.

The following lemma establishes a basic inequality containing the estimates worker $k$ returns for its raw features and the global ridge regression solution.
\begin{lem} \label{lem:workerk_soln}
Setting $\solnrr_{k, \text{raw}}$ as the first $\dimsk$ elements (i.e. those corresponding to the raw features) of the vector which minimizes the ridge problem of worker $k$
$$
\underset{\nrm{\soln_k}^2 \leq t}{\min} n^{-1} \nrm{\y - \feats_k \soln_k}^2 = 
\underset{\nrm{\soln_{\text{raw}}}^2 +\nrm{\soln_{k, \text{rp}}}^2 \leq t}{\min} n^{-1} \nrm{\y - \At_k \soln_{\text{raw}} - \tilde{\At}_k  \soln_{k, \text{rp}}}^2
$$
in a compressed space, we have
$$
\underset{\nrm{\solnrr_{k, \text{raw}}}^2 + \nrm{ \soln_{\minusk}}^2 \leq t}{\min} n^{-1} \nrm{\y - \At_k \solnrr_{k, \text{raw}} - \At_{\minusk}  \soln_{\minusk}}^2
\leq 
\underset{\nrm{\soln}^2 \leq t}{\min} n^{-1} \nrm{\y - \At \soln}^2 + \Delta
$$
where $\solnrr_{k, \text{raw}}$ is now fixed and the design matrices are standardized. $\Delta$ is the difference between the global objective and the objective of worker $k$ at their respective minimal values, i.e.
$$
\Delta = n^{-1} \nrm{\y - \feats_k \solnrr_k}^2 - n^{-1} \nrm{\y - \At \solnrr}^2.
$$
\end{lem}
\begin{proof} 
For the same value of $t$ we have
\begin{eqnarray} 
\nonumber
\underset{\nrm{\soln}^2 \leq t}{\min} n^{-1} \nrm{\y - \At \soln}^2 &\leq & 
\underset{\nrm{\soln_k}^2 \leq t}{\min} n^{-1} \nrm{\y - \feats_k \soln_k}^2 \\
\small
\underset{\tiny \nrm{\soln_{\text{raw}}}^2 + \nrm{ \soln_{\minusk}}^2 \leq t}{\min} n^{-1} \nrm{\y - \At_k \soln_{\text{raw}} - \At_{\minusk}  \soln_{\minusk}}^2 &\leq & 
\small
\underset{\tiny \nrm{\soln_{\text{raw}}}^2 + \nrm{ \soln_{k, \text{rp}}}^2 \leq t}{\min} n^{-1} \nrm{\y - \At_k \soln_{\text{raw}} - \tilde{\At}_k  \soln_{k, \text{rp}}}^2
\nonumber\\
\label{analysis:ridge_ineq1}
\end{eqnarray}
where the inequality follows from the fact that the original space is larger such that $\y$ can be approximated better. In particular, we have more features in $  \At_{\minusk} $ than in $\tilde{\At}_k $ and the latter lies in the column space of the former.

Replacing $ \At_k \soln_{\text{raw}} $ with $\At_k \solnrr_{k, \text{raw}}$ on the left hand side in inequality \eqref{analysis:ridge_ineq1} where $\solnrr_{k, \text{raw}}$ is part of the solution to the right hand side we obtain
\begin{eqnarray}\nonumber
\underset{\nrm{\solnrr_{k, \text{raw}}}^2 + \nrm{ \soln_{\minusk}}^2 \leq t }{\min} n^{-1} \nrm{\y - \At_k \solnrr_{k, \text{raw}} - \At_{\minusk}  \soln_{\minusk}}^2 &\leq & 
\underset{\nrm{\soln_{\text{raw}}}^2 + \nrm{ \soln_{k, \text{rp}}}^2 \leq t}{\min} n^{-1} \nrm{\y - \At_k \soln_{\text{raw}} - \tilde{\At}_k  \soln_{k, \text{rp}}}^2
\end{eqnarray}
which holds due to the same argument as above.

Lastly, we rewrite the right hand side in terms of the original objective
\begin{eqnarray}
\nonumber
\small
\underset{\tiny \nrm{\solnrr_{k, \text{raw}}}^2 + \nrm{ \soln_{\minusk}}^2 \leq t }{\min} n^{-1} \nrm{\y - \At_k \solnrr_{k, \text{raw}} - \At_{\minusk}  \soln_{\minusk}}^2 &\leq &   
\small
\underset{\tiny \nrm{\soln_{\text{raw}}}^2 + \nrm{ \soln_{\minusk}}^2 \leq t }{\min} n^{-1} \nrm{\y - \At_k \soln_{\text{raw}} - \At_{\minusk}  \soln_{\minusk}}^2 + \Delta
\end{eqnarray} 
where $\Delta$ accounts for the difference, i.e.
\begin{equation} \label{eq:delta_explicit}
\Delta = \underset{\nrm{\soln_k}^2 \leq t}{\min} n^{-1} \nrm{\y - \feats_k \soln_k}^2 - \underset{\nrm{\soln}^2 \leq t}{\min} n^{-1} \nrm{\y - \At \soln}^2.
\end{equation}
\end{proof}

Recall that in {\bf step 4.} of \loco we construct the design matrix of worker $k$ by concatenating the random features from the remaining blocks. The following lemmas use ideas from \citep{lu:2013} to quantify the effect of the random projections.
First, Lemma \ref{cor:concatRF} establishes a bound on the spectral norm between the design matrix of worker $k$ and the global design matrix. We use this fact in Lemma \ref{cor:kernel_bound_concat} to upper bound the largest eigenvalue and lower bound the $r^{th}$-largest eigenvalue of the covariance matrix of the projected data in terms of the respective eigenvalues of the data covariance matrix. Recall that $r$ is the rank of $\Xt$.

\begin{lem}[Concatenating random features]\label{cor:concatRF}
Consider the singular value decomposition $\Xt = \Ut \Dt \Vt\tr$ where $\Ut \in \R^{\samp\times r}$ and $\Vt \in \R^{p\times r}$ have orthonormal columns and $\Dt \in \R^{r\times r}$ is diagonal; $r=\text{rank}(\At)$. In addition to the raw features, let $\feats_k \in \R^{n\times(\dimsk + (K-1)\dimsks)}$ contain random features which result from concatenating the $K-1$ random projections from the other workers. Furthermore, assume without loss of generality that the problem is permuted so that the raw features of worker $k$'s problem are the first $\dimsk$ columns of $\At$ and $\feats_k$.
Finally, let
$$
\Theta_C = \begin{bmatrix}
\It_\tau & 0 & \ldots  & 0 \\
0 & \RP_1 & 0   & \vdots\\ 
\vdots & \ldots & \ddots & 0\\ 
0 & \ldots & \ldots  & \RP_{K-1}\\ 
\end{bmatrix}
\in \R^{p\times(\dimsk + (K-1)\dimsks)}
$$
such that $\feats_k = \Xt\Theta_C.$

With probability at least $1-\br{\delta + \frac{\dims-\dimsk}{e^r}}$

$$
\nrm{\Vt\tr\Theta_C \Theta_C\tr\Vt - \Vt\tr\Vt} \leq \sqrt{\frac{c \log(2 r /\delta) r}{(\blocks-1)\dimsks}}.
$$
\end{lem}
\begin{proof}
Let $\Vt_k$ contain the first $\dimsk$ rows of $\Vt$ and let $\Vt_{k'}$ contain the rows of $\Vt$ which are multiplied by the rows of $\Theta_C$ containing $\RP_{k'}$. Decompose the matrix products as follows
$$ 
\Vt\tr\Vt =   \Vt_k\tr\Vt_k + \sum_{k'\neq k}  \Vt_{k'}\tr\Vt_{k'} 
\; \; \; \text{ and } \; \; \;
\Vt\tr\Theta_C \Theta_C\tr\Vt =  \Vt_k\tr\Vt_k+ \sum_{k'\neq k}  \tilde{\Vt}_{k'}\tr\tilde{\Vt}_{k'}.
$$
With  $\tilde{\Vt}_{k'}\tr = \Vt_{k'}\tr \RP_{k'}$ we have
$$
\nrm{\Vt\tr\Theta_C \Theta_C\tr\Vt - \Vt\tr\Vt} = \nrm{\sum_{k'\neq k}\br{\Vt_{k'}\tr \RP_{k'} \RP_{k'}\tr \Vt_{k'} - \Vt_{k'}\tr\Vt_{k'}} }  
$$
Since $\Theta_C$ is an orthogonal matrix, from Lemma 3.3 in \citep{Tropp:2010uo} and Lemma \ref{lem:concatenated_row}, concatenating $(K-1)$ independent SRHTs from $\dimsk$ to $\dimsks$ is equivalent to applying a single SRHT from $\dims-\dimsk$ to $(K-1)\dimsks$. 
Therefore we can simply apply Lemma  \ref{lem:srht} to the above to obtain the result.
\end{proof}

\begin{lem}\label{cor:kernel_bound_concat}
Let $\feats_k \in \R^{n\times(\dimsk + (K-1)\dimsks)}$ be as defined above.
For $\At$ with $r=\text{rank}(\At)$ and with probability at least $1-  (\delta + \frac{\dims-\dimsk}{e^r})$
$$
(1-\rho) \At\At\tr \preceq \bar{\At}_k\bar{\At}_k\tr \preceq (1+\rho) \At\At\tr
$$
where $\rho = C\sqrt{\frac{r\log(2r/\delta)}{(\blocks-1)\dimsks}}$. Here, ${\bf A} \preceq {\bf B}$ means that $({\bf B} - {\bf A})$ is a positive semi-definite matrix. 
\end{lem}

\begin{proof}
The proof is analogous to the proof of Corollary 1 in \citep{lu:2013}. Using the singular value decomposition $\Xt = \Ut \Dt \Vt\tr$, we have 
$$
\feats_k \feats_k\tr=  \Ut \Dt \Vt\tr \Theta_C \Theta_C\tr \Vt \Dt \Ut\tr 
$$
where $\Theta_C$ is as defined in Lemma \ref{cor:concatRF}. Lemma \ref{cor:concatRF} implies that 
$$
(1-\rho) \Vt\tr\Vt \preceq \Vt\tr \Theta_C \Theta_C\tr \Vt \preceq (1+\rho) \Vt\tr\Vt
$$
with probability at least $1-\br{\delta + \frac{\dims-\dimsk}{e^r}}$ and $ \rho = C\sqrt{\frac{r\log(2r/\delta)}{(\blocks-1)\dimsks}}$. The result follows by multiplying with $\Ut \Dt$ and $\Dt \Ut\tr$.
\end{proof}


\begin{lem}\label{lem:bound_delta}
Let $\risk(\Xt\solnrr)$ denote the risk of the global ridge estimate and let $\risk(\feats_k\solnrr_k)$ denote the risk of the ridge estimate worker $k$ returns. For $\feats_k \in \R^{n\times(\dimsk + (\blocks-1)\dimsks)}$ with probability at least $1-\br{\delta + \frac{\dims-\dimsk}{e^r}}$, the expected difference in the risk is bounded
$$
\risk(\feats_k\solnrr_k)  - \risk(\Xt\solnrr) = \bE_\varepsilon \Delta \leq \br{\frac{1}{(1-\rho)^{2}}-1} \risk(\Xt\solnrr)
$$
where $\rho = C \sqrt{\frac{r \log(2r / \delta) }{(\blocks-1)\dimsks}}$ and $\Delta$ is the difference between the objectives \eqref{analysis:global_ridge} and \eqref{analysis:workerk_ridge} at their minimal values.
\end{lem}

\begin{proof}
As Lemma \ref{cor:kernel_bound_concat} holds for $\feats_k \in \R^{n\times(\dimsk + (\blocks-1)\dimsks)}$, Lemma \ref{lem:risk_inflation} applies with probability at least $1-\br{\delta + \frac{\dims-\dimsk}{e^r}}$. Thus, we have
$$
\risk(\feats_k \solnrr_k) \leq (1-\rho)^{-2} \risk(\Xt\solnrr) = \risk(\Xt\solnrr) + \br{\frac{1}{(1-\rho)^{2}}-1} \risk(\Xt\solnrr) 
$$
where $\rho = C \sqrt{\frac{r \log(2r / \delta) }{(\blocks-1)\dimsks}}$. $\Delta$ was introduced in eq. \eqref{eq:delta_explicit} in Lemma \ref{lem:workerk_soln} as the difference between the objectives \eqref{analysis:global_ridge} and \eqref{analysis:workerk_ridge} at their minimal values. Taking the expectation of this difference w.r.t. $\varepsilon$ coincides with the difference between the risk functions. Therefore, we have
$$
\risk(\feats_k \solnrr_k) -  \risk(\Xt\solnrr) =  \bE_\varepsilon \Delta.
$$
\end{proof}

\begin{lem}\label{lem:gradient}
Let $\nabla f(\solnrr)$ denote the gradient of $f$ at $\solnrr$ and let $\lambda$ be the regularization parameter in the penalized formulation of the objective. 
Then we have that 
\begin{equation}\label{eq:gradient_lambda}
\nabla f(\solnrr) \geq - 2 \lambda \solnrr.
\end{equation}
\end{lem}

\begin{proof}
Without loss of generality rotate $\Xt$ to the PCA coordinate system such that 
$$
\boldSigmar = \frac{1}{n}\Xtr\tr\Xtr = \diag(\lambda_1, \ldots, \lambda_{\min(n, p)}, 0, \ldots, 0)
$$
with $\lambda_1, \ldots, \lambda_{\min(n,p)}$ being the non-zero eigenvalues of $\boldSigmar$ in decreasing order. The subscript indicates that $\Xtr$ is the rotated design matrix.

Consider the estimator
$$
\solnrls = \lim_{\lambda \rightarrow 0} \solnrrr
$$
where $\solnrrr$ contains the ridge regression estimates in the rotated space.
Since the OLS estimator is not uniquely defined if there are zero eigenvalues, $\solnrls$ is the least squares solution with minimal $\ell_2$ norm. 

In this orthogonal setting, the ridge estimate for the $j$-th coefficient ${\solnrrr}_j$ has a simple relation to the corresponding estimate $ {\solnrls}_j$
$$
{\solnrrr}_j = \frac{\lambda_j}{\lambda_j + \lambda} {\solnrls}_j.
$$
Furthermore, the loss can be expressed as
$$
f(\solnrrr) 
= \sum_{j = 1}^\dims \lambda_j \br{1 - \frac{\lambda_j}{\lambda_j + \lambda}}^2 (\solnrls_j)^2 + c'
= \sum_{j = 1}^\dims \lambda_j \br{{\solnrls}_j - {\solnrrr}_j}^2 + c'
$$
where $c'$ is a constant and the gradient follows as
$$
\nabla f(\solnrrr) 
= - \sum_{j = 1}^\dims 2 \lambda_j \br{{\solnrls}_j - {\solnrrr}_j}
= - \sum_{j = 1}^\dims \frac{2 \lambda \lambda_j}{\lambda_j + \lambda} {\solnrls}_j 
= - 2 \lambda \solnrrr.
$$
As $f$ was not changed by rotating $\Xt$, rotating $\solnrrr$ back to retrieve $\solnrr$ does not change the gradient. Therefore, we have
\begin{equation*}
\nabla f(\solnrr) = - 2 \lambda \solnrr.
\end{equation*}
\end{proof}

\begin{lem} \label{lem:bound_lambda}
Let $\lambda$ be the regularization parameter in the penalized formulation of the objective and let $\lambda_J$ denote the $J$-th largest eigenvalue of the covariance matrix.
Then under assumptions (A1) through (A3) there exists a $n_0$ such that for all $n\ge n_0=n_0(\gamma, \sigma^2, J, \lambda_J, L_J, \boundconst^2)$, with probability at least $1-\gamma$, 
\[ \lambda \ge \frac{\boundconst}{5} \lambda_J.\]
\end{lem}

\begin{proof}
Let $\Xtr$ denote the design matrix after rotating $\Xt$ to the PCA coordinate system and let $\solnrrr$ denote the ridge estimate in this space. Then we can express ${\solnrrr}_j$ as
$$
{\solnrrr}_j = \frac{\lambda_j}{\lambda_j + \lambda} \br{\solnrstar_j + \noisevar_j}
$$
where $\solnrstar$ is the true parameter vector in the PCA coordinate system and 
$$\noisevar_j = \dfrac{(\Xtr)_j\tr\varepsilon}{(\Xtr)_j\tr(\Xtr)_j} =  \dfrac{(\Xtr)_j\tr\varepsilon}{\samp \lambda_j}.$$

Define \[ L_J := \sum_{j=1}^J (\solnrstar_j)^2.\]

As $\solnrr$ is the solution to the constrained optimization problem in eq. \eqref{analysis:global_ridge}, we have that $ \nrm{\solnrrr}^2 \leq t$. As $J < p$ and using the relation between $\solnrstar$ and $\solnrrr$ yields
\[ \sum_{j=1}^J \Big( \frac{\lambda_j}{\lambda_j+\lambda}\Big) ^2 (\solnrstar_j+\noisevar_j)^2\le t.\]
Hence
\begin{equation} \label{eq:sq} \Big(\sum_{j=1}^J \Big( \frac{\lambda_j}{\lambda_j+\lambda}\Big) ^2 (\solnrstar_j)^2\Big)^{1/2} \le (t)^{1/2}+ \Big(\sum_{j=1}^J \Big( \frac{\lambda_j}{\lambda_j+\lambda}\Big) ^2 \noisevar_j^2 \Big)^{1/2}.\end{equation}
Using monotonicity of the eigenvalues, the left hand side is bounded by
\[ \sum_{j=1}^J \Big( \frac{\lambda_j}{\lambda_j+\lambda}\Big) ^2 (\solnrstar_j)^2\;\ge\;  \Big(\frac{\lambda_J}{\lambda_J+\lambda}\Big)^2 \sum_{j=1}^J (\solnrstar_j)^2.\]
Using assumption (A2), we thus have from (\ref{eq:sq}) 
\begin{equation} \label{eq:sq2}  \Big( \Big(\frac{\lambda_J}{\lambda_J+\lambda}\Big)^2 L_J\Big)^{1/2}  \le  \big( (1-\boundconst) L_J\big)^{1/2} + \Big(\sum_{j=1}^J \Big( \frac{\lambda_j}{\lambda_j+\lambda}\Big) ^2 \noisevar_j^2 \Big)^{1/2}.\end{equation}
If we can show that there exists some $n_0$ such that for all $n\ge n_0$ with probability at least $1-\gamma$,
\begin{equation}\label{eq:toshow} \sum_{j=1}^J \Big( \frac{\lambda_j}{\lambda_j+\lambda}\Big) ^2 \noisevar_j^2  \le  \frac{\boundconst^2}{16} L_J,\end{equation}
then, from (\ref{eq:sq2}), with probability at least $1-\gamma$,
\begin{align*}  \Big( \Big(\frac{\lambda_J}{\lambda_J+\lambda}\Big)^2 L_J\Big)^{1/2}  & \le \big( (1-\boundconst) L_J\big)^{1/2} + \big(  (\boundconst^2 /16) L_J  \big)^{1/2} ,\\ 
\mbox{and hence   } \qquad  \Big(\frac{\lambda_J}{\lambda_J+\lambda}\Big)^2  &\le (1-\boundconst) + 2\sqrt{(1-\boundconst)}\sqrt{\boundconst^2/16} + (\boundconst^2/16) \\ 
&\le  (1-\boundconst) + \boundconst/2+ \boundconst^2/16 \le 1- \frac{ 7\boundconst}{16}.   \end{align*}
And, if (\ref{eq:toshow}) is true, there thus exists a $n_0$ such that with probability at least $1-\gamma$,
\begin{align*}
\lambda &\ge \Big( \frac{1}{\sqrt{1-7\boundconst/16}}-1 \Big) \lambda_J \\
&\ge \frac{\boundconst}{5} \lambda_J.
\end{align*}

It thus remains to show (\ref{eq:toshow}).
The mean of 
\[ \frac{\lambda_j}{\lambda_j+\lambda}  \noisevar_j \]
vanishes and the variance is given, for all $j=1,\ldots,J$ by
\[  \frac{\sigma^2 \lambda_j }{n(\lambda+\lambda_j)^2} \le \frac{\sigma^2 }{n \lambda_J} . \]
Hence, for all $j=1,\ldots,J$ and $a>0$,
\[ 
\Prob \cb{  \big(\frac{\lambda_j}{\lambda_j+\lambda}\big)^2  \noisevar_j^2 \ge a^2 }  \le \frac{1}{a^2} \frac{\sigma^2}{n\lambda_J} .
\]
Thus, using a Bonferroni bound over $j=1\ldots,J$,
\begin{align*}
\Prob \cb{  \sum_{j=1}^J \big(\frac{\lambda_j}{\lambda_j+\lambda}\big)^2  \noisevar_j^2 \ge a^2 }  &\le 
\Prob \cb{  J \max_{j=1,\ldots,J}  \big(\frac{\lambda_j}{\lambda_j+\lambda}\big)^2  \noisevar_j^2 \ge a^2 } \\
&\le J \max_{j=1,\ldots,J} \Prob \cb{  \big(\frac{\lambda_j}{\lambda_j+\lambda}\big)^2  \noisevar_j^2 \ge a^2/J } \\
& \le \frac{J^2}{a^2} \frac{\sigma^2}{n\lambda_J}.
\end{align*}
Hence, with probability at least $1-\gamma$, 
\begin{equation} \sum_{j=1}^J \Big( \frac{\lambda_j}{\lambda_j+\lambda}\Big) ^2 \noisevar_j^2  \le  \frac{J^2}{\gamma} \frac{\sigma^2}{n\lambda_J},\end{equation}
If choosing 
\[ n_0\; := \; \frac{16 J^2 \sigma^2 }{\gamma \lambda_J L_J \boundconst^2},\]
then, with probability at least $1-\gamma$,
\begin{equation} \sum_{j=1}^J \Big( \frac{\lambda_j}{\lambda_j+\lambda}\Big) ^2 \noisevar_j^2  \le  \frac{\boundconst^2}{16} L_J,\end{equation}
which shows (\ref{eq:toshow}) and thus completes the proof.
\end{proof}

\vspace{0.5cm}


\noindent We are now ready to present the proof of our main theorem.
\paragraph{Proof of Theorem \ref{thm_ridge}.}
Letting $\solnrr_{k, \text{raw}}$ denote the first $\dimsk$ elements of $\solnrr_k$, we set
$$
\tilde\soln' = \underset{\nrm{\solnrr_{k, \text{raw}}}^2 + \nrm{ \soln_{\minusk}}^2 \leq t }{\arg\min} n^{-1} \nrm{\y - \At_k \solnrr_{k, \text{raw}} - \At_{\minusk}  \soln_{\minusk}}^2,
$$
where $\solnrr_{k, \text{raw}} $ is fixed and the minimization is over $\soln_{\minusk}$ i.e. $\tilde\soln' = \sq{\solnrr_{k,\text{raw}}; \tilde\soln_{\minusk}}$. 

From Lemma \ref{lem:workerk_soln}, it follows that $f(\tilde\soln') - f(\solnrr) \leq \Delta $ and as $f(\soln)$ is convex, we also have that
\begin{equation}
f(\tilde\soln') \geq f(\solnrr) + \nabla f(\solnrr)^\top (\tilde\soln' - \solnrr).
\label{analysis:ridge_convex}
\end{equation}
Under Assumption \ref{assn:bound_t} , using eq. \eqref{eq:gradient_lambda} from Lemma \ref{lem:gradient} and the lower bound on $\lambda$ from Lemma \ref{lem:bound_lambda},
the difference between the components of interest is bounded by
\begin{eqnarray*}
\Delta 
&\geq& f(\tilde\soln') - f(\solnrr) 
\geq \frac{\boundconst}{5} \lambda_J \nrm{\solnrr - \tilde\soln'}^2 \\
&\geq&  \frac{\boundconst}{5} \lambda_J \nrm{\solnrr_\text{raw} -\solnrr_{k,\text{raw}}}^2
\end{eqnarray*}
with probability $1 - \gamma$.
Using the expression for the expectation of $\Delta$ from Lemma \ref{lem:bound_delta}, 
$$
\bE_\varepsilon(\nrm{\solnrr_\text{raw} -\solnrr_{k,\text{raw}}}^2) \leq 
\frac{5}{\boundconst \lambda_J} \br{\frac{1}{(1-\rho)^{2}} -1} \risk(\Xt\solnrr)
$$
with probability $1-\br{\delta + \frac{\dims-\dimsk}{e^r} + \gamma}$.
Lastly, we use this expression to find a bound for the difference between the full ridge solution $\solnrr$ and the estimate $\solnloco$ returned by \loco

$$
\bE_\varepsilon(\nrm{\solnrr -\solnloco}^2) \leq 
\frac{5 K}{\boundconst \lambda_J} \br{\frac{1}{(1-\rho)^{2}} -1} \risk(\Xt\solnrr)
$$
with probability $1- \xi = 1- \blocks \br{\delta + \frac{\dims-\dimsk}{e^r} + \gamma}$.

\begin{rem}[Discussion of bound]\label{rem:discussion_bound}
Combining eq. \eqref{eq:gradient_lambda} and the lower bound on $\lambda$ from Lemma \ref{lem:bound_lambda} shows that the gradient is bounded away from zero under Assumption \ref{assn:bound_t}. In figure \ref{fig:bound}, $\solnlsmin$ denotes the least squares solution with minimal $\ell_2$ norm. Then, under (A1) through (A3) we can ensure that the distance $d$ shown in figure \ref{fig:bound} does not go to zero which would translate into a weaker bound.


\begin{figure*}[htp]
    \center
    \includegraphics[width=0.5\textwidth, keepaspectratio=true]{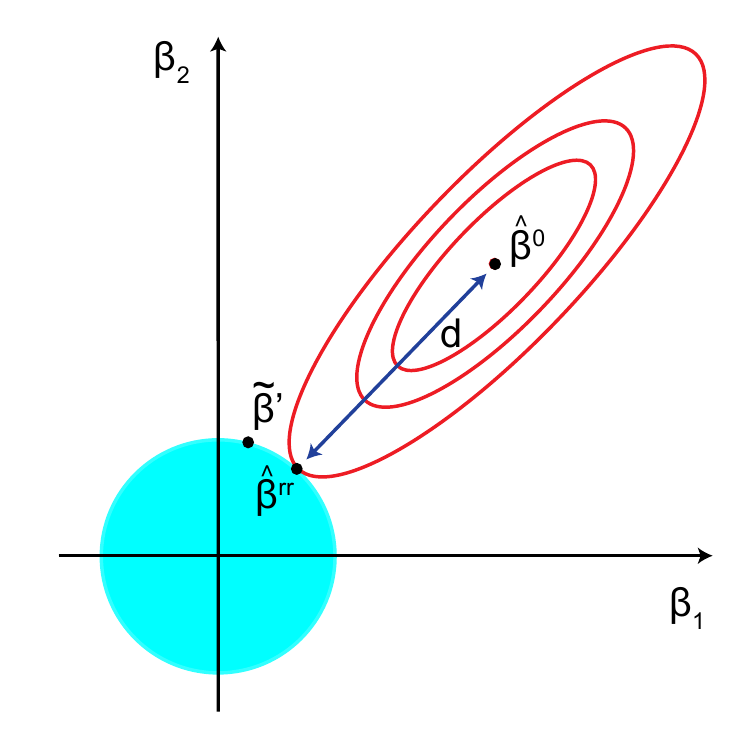}
    \caption{Illustration of bound. \label{fig:bound}}
\end{figure*}
\end{rem}

\subsubsection*{Bounding the error when $\nrm{\solnlsmin}^2 \leq t$ and further considerations.} 
Above, we considered the case where the constraint $t$ was active, i.e. where regularization was needed to achieve best predictive accuracy. In this setting, the least squares solution with minimal $\ell_2$ norm $\solnlsmin$ has a larger squared $\ell_2$ norm than $t$.
For completeness, consider the case where $\nrm{\solnlsmin}^2 \leq t $. Then the constraint $t$ is large enough such that the ridge estimate coincides with ${\solnlsmin}$ and $\nabla f(\solnrr) =0$. 
Then
$$
f(\tilde\soln') - f(\solnrr) =  (\tilde\soln' - \solnrr)\tr \frac{\Xt \tr \Xt}{n} (\tilde\soln' - \solnrr)
$$
where we made use of the fact that the second-order Taylor expansion holds exactly for the squared error loss. If $\tilde{\soln'} = \solnrr$, worker $k$ estimates the coefficients of the raw features optimally such that we do not have to consider this case further. If  $\tilde{\soln'} \neq \solnrr$, the following result relates $\nrm{\At (\tilde\soln' - \solnrr) }^2 $ to $\nrm{\tilde\soln' - \solnrr}^2 $.

We first make an assumption on the smallest eigenvalue of the covariance matrix $\boldSigma = \frac{1}{n} \Xt\tr\Xt$.
\begin{assn}[Restricted minimum eigenvalue condition]
$$
\min_\psi \cb{ \frac{(\psi - \gamma)^\top \boldSigma (\psi - \gamma)}{\nrm{\psi - \gamma}^2}; \text{ such that } \nrm{\psi} = 1,  \gamma = \frac{\solnrr}{\nrm{\solnrr}} } \geq \varphi > 0.
$$
\label{assn:smallest_ev}
\end{assn}

\begin{thm}
If for some constant $\varphi > 0$, $$\nrm{\tilde\soln' - \solnrr}^2 \leq (\tilde\soln' - \solnrr)^\top \boldSigma (\tilde\soln' - \solnrr) / \varphi,$$ 
we can upper bound $\nrm{\tilde\soln' - \solnrr}^2 $ by $\Delta / \varphi$ so that
$$
\bE_\varepsilon(\nrm{\solnrr_\text{raw} -\solnrr_{k,\text{raw}}}^2) \leq 
\frac{1}{\varphi} \br{\frac{1}{(1-\rho)^{2}} -1} R(\Xt\solnrr)
$$
holds with probability at least $1-\br{\delta + \frac{\dims-\dimsk}{e^r}}$ and $\rho = C \sqrt{\frac{r \log(2r / \delta) }{(\blocks-1)\dimsks}}$.
\end{thm}

\begin{proof}
If Assumption \ref{assn:smallest_ev} holds, then 
$$
\frac{\Delta}{\varphi} \geq 
(\tilde\soln' - \solnrr)^\top \boldSigma (\tilde\soln' - \solnrr)/ \varphi \geq 
\nrm{\tilde\soln' - \solnrr}^2  \geq
\nrm{\solnrr_{k,\text{raw}}  -  \solnrr_\text{raw}}^2
$$
where $\varphi$ is the smallest eigenvalue of $\boldSigma$, and 
$$
\nrm{\solnrr_{k,\text{raw}}  -  \solnrr_\text{raw}}^2 \leq \frac{\Delta}{\varphi}.
$$
Using the expression for the expectation of $\Delta$ from Lemma \ref{lem:bound_delta} yields the result.
\end{proof}

\begin{rem}\label{rem:eigenvalue}
Assumption \ref{assn:smallest_ev} implies that the smallest eigenvalue of the $\boldSigma $ is bounded away from zero. Of course, this would not be satisfied if $p > n$. On the other hand, if $p > n$ cross validation would always yield a value for $t$ such that $\nrm{\solnlsmin}^2 > t$ because the least squares solution has bad statistical properties in this setting and regularization is needed to achieve best predictive accuracy. Furthermore, if $n \geq p$ and the covariance matrix is close to being singular, we would again choose a constraint such that $\nrm{\solnlsmin}^2 > t $. Therefore, we can assume that the smallest eigenvalue of the covariance matrix is bounded away from zero in all relevant cases, namely if cross validation yields a value of t such that $\nrm{\solnlsmin}^2 \leq t $.
\end{rem}

\subsection{Ordinary Least Squares} \label{sec:supp_ols}
For OLS regression, $f(\soln) = \nrm{\y - \At\soln}^2$ is the squared error loss while there is no regularizer. Here, $\solnls$ denotes the solution which results from solving the OLS problem using the original data while each worker solves least squares using $\feats_k$ instead of $\At$ and returns $\solnls_k$.

\begin{defn}
Let $\widehat{\mathbf{z}}_j$ denote the residual which results from regressing the feature vector $j$, $\At^{j}$, onto the remaining features $\At^{-j}$
$$
\widehat{\mathbf{z}}_j = \At^{j} - \At^{-j} \widehat{\gamma} \; \; \;  \text{   where   } \; \; \;   \widehat{\gamma} = \underset{\gamma}{\arg\min} \nrm{ \At^{j} - \At^{-j} \gamma}^2
$$
and let $\tilde{\mathbf{z}}_j$ denote the residual which results from regressing feature $j$, $\At^{j}$, onto the randomized approximation of the remaining features $\feats_k^{-j} $
$$
\tilde{\mathbf{z}}_{j} = \At^{j}  - \feats_k^{-j} \tilde{\gamma} \; \; \; \text{   where   } \; \; \;  \tilde{\gamma} = \underset{\gamma_k}{\arg\min} \nrm{ \At^{j}  - \feats_k^{-j} \gamma_k}^2.
$$
\label{def:ols_auxresiduals}
\end{defn}

The following result relies on Lemma \ref{lem_kaban_general} by \citep{kaban:2014} which can be found in \S \ref{sec:supporting_lemma}.

\begin{thm}
Let $\widehat{\mathbf{z}}_j$, $\tilde{\mathbf{z}}_j$ and $\widehat{\gamma}$ be as given in Definition \ref{def:ols_auxresiduals} and let $\boldSigma = \At^{-j \top}\At^{-j}/n$. Furthermore, let all other quantities be as defined in Lemma \ref{lem_kaban_general}. Then the difference between the residuals is bounded by
$$
\frac{1}{n} \bE_\Pi (\nrm{\widehat{\mathbf{z}}_j -\tilde{\mathbf{z}}_j}^2) \leq  \frac{1}{\dimsks} \nrm{\widehat{\gamma} }^2_{\boldSigma + \trace{\boldSigma} \mathbf{I}_{p-1} + \kappa \sum_{i = 1}^{p-1} e_i B_i}
$$ 
where $\nrm{\mathbf{u}}_\mathbf{M} = \sqrt{\mathbf{u}^\top \mathbf{M} \mathbf{u} }$ is the Mahalanobis norm.
\label{thm_ols}
\end{thm}

\begin{cor}
As the $j$-th OLS regression coefficient can be expressed as 
$$
\solnls_j = \frac{\widehat{\mathbf{z}}_j \cdot \y }{\widehat{\mathbf{z}}_j\cdot \At^{j}}
\; \; \; \text{ resp. } \; \; \; 
\solnls_{k, j} = \frac{ \tilde{\mathbf{z}}_{j} \cdot \y }{\tilde{\mathbf{z}}_{j} \cdot \At^{j}},
$$
the bound on the difference between $\widehat{\mathbf{z}}_j$ and $\tilde{\mathbf{z}}_j$ implies that the difference between $\solnls_j$ and $\solnls_{k, j}$ is bounded as well. Therefore, the estimate for the raw feature $j$ returned by worker $k$, $\solnls_{k, j}$, is sufficiently close to the global estimate $\solnls_j$.
\label{cor:implication_OLS_coeff}
\end{cor}

In order to present the proof of Theorem \ref{thm_ols}, we need the following lemmata. 

\begin{lem}
Consider the OLS regression problem in the original space
\begin{equation}
\solnls =\underset{\soln}{\arg\min} f(\soln) = \underset{\soln}{\arg\min} \nrm{\y - \At\soln}^2,
\label{analysis:ols_orig}
\end{equation}
yielding the fitted values $\widehat{\y} = \At \solnls$. Furthermore, let $\feats_k$ be the design matrix of the least squares problem worker $k$ solves. Then the squared $\ell_2$ norm of the residual can be decomposed as follows
\begin{equation}
\nrm{\y - \widehat{\y}}^2 = 
\underbrace{\nrm{\mathbf{P_{\bot \At}}(\y - \widehat{\y})}^2}_{\neq 0}
+ \underbrace{\nrm{\mathbf{P}_{\| \At \bot \feats_k}(\y - \widehat{\y})}^2}_{ = 0}
+ \underbrace{\nrm{\mathbf{P}_{\|\feats_k}(\y - \widehat{\y})}^2,}_{= 0}
\label{analysis_ols_eq_residuals}
\end{equation}
\label{lem_residuals}
where $\mathbf{P_{\| \At}} = \At(\At\tr\At)^{-1}\At\tr$ denotes the orthogonal projection onto the column space of the matrix $ \At$ and $\mathbf{P_{\bot \At}} = \mathbf{I} - \mathbf{P_{\| \At}}$ is the projection onto the orthogonal complement of the column space of $\At$. 
\end{lem}

\begin{proof}
The decomposition follows from the orthogonality of the considered spaces. Note that $\feats_k$ lies in the column space of $\At$ since the random features are linear combination of the raw features not in $\dimset_k$ and the features in $\dimset_k$ are contained in both matrices. The first summand in eq. \eqref{analysis_ols_eq_residuals} is the part of $\y$ that cannot be accounted for in the space spanned by the original design matrix $\At$. Therefore, it is the error that the OLS fit cannot avoid to incur. The remaining two terms are zero due to the definition of the OLS estimator
\begin{eqnarray*}
\nrm{\mathbf{P}_{\| \At \bot \feats_k}(\y - \widehat{\y})}^2 + \nrm{\mathbf{P}_{\|\feats_k}(\y - \widehat{\y})}^2 = 
\nrm{\mathbf{P}_{\| \At}(\y - \widehat{\y})}^2 &=& 
\nrm{ \At(\At\tr\At)^{-1}\At\tr \y - \widehat{\y}}^2  \\
&=& \nrm{ \At\solnls - \widehat{\y}}^2 = 0.
\end{eqnarray*}
\end{proof}

\begin{lem}
Consider the OLS regression problem in the orginal space as given in eq. \eqref{analysis:ols_orig} and in the compressed space
$$
\solnls_k =\underset{\soln_k}{\arg\min} f_k(\soln_k) = \underset{\soln_k}{\arg\min} \nrm{\y - \feats_k\soln_k}^2,
$$ yielding the fitted values $\tilde{\y} = \feats_k \solnls_k$. The squared $\ell_2$ norm of the difference between $\widehat{\y}$ and $\tilde{\y}$ can be expressed as
$$
\nrm{\widehat{\y} - \tilde{\y}}^2 = \nrm{ \At \solnls - \feats_k \solnls_k}^2  = \nrm{\mathbf{P}_{\| \At \bot \feats_k}( \y - \tilde{\y})}^2 = \nrm{\mathbf{P}_{\| \At \bot \feats_k}(\y)}^2
$$
which corresponds to the approximation error which results from fitting the model in the compressed space instead of the original space.
\label{lem:diff_residuals_ols}
\end{lem}
\begin{proof}
According to Lemma \ref{lem_residuals} the residual of the fit in the compressed domain can be decomposed as follows
\begin{equation}
\nrm{\y - \tilde{\y}}^2 = 
\underbrace{\nrm{\mathbf{P_{\bot \At}}(\y -  \tilde{\y})}^2}_{\neq 0}
+ \underbrace{\nrm{\mathbf{P}_{\| \At \bot \feats_k}(\y -  \tilde{\y})}^2}_{ \neq 0}
+ \underbrace{\nrm{\mathbf{P}_{\|\feats_k}(\y -  \tilde{\y})}^2.}_{= 0}
\label{analysis_ols_eq_residuals_compressed}
\end{equation}
Now, the first and the second term form the part of $\y$ that cannot be accounted for in the space spanned by the projected design matrix while the third term vanishes due to the definition of the least squares estimator. Thus, we see that the difference between eq. \eqref{analysis_ols_eq_residuals} and eq. \eqref{analysis_ols_eq_residuals_compressed} is given by the second term in eq. \eqref{analysis_ols_eq_residuals_compressed} and due to the orthogonality structure, we have
$$
\nrm{\widehat{\y} - \tilde{\y} }^2 = \nrm{\At \solnls - \feats_k \solnls_k}^2  = \nrm{\mathbf{P}_{\| \At \bot \feats_k}( \y - \tilde{\y})}^2 = \nrm{\mathbf{P}_{\| \At \bot \feats_k}(\y)}^2.
$$
This expression corresponds to the approximation error which results from fitting the model in the compressed space instead of the original space.
\end{proof}

\begin{lem}
Let $\widehat{\y}$ denote the fitted values resulting from the OLS problem in the original space and let $\tilde{\y}$ be the fitted values resulting from the OLS problem in the compressed space of worker $k$. Furthermore, let all other quantities be as defined in Lemma \ref{lem_kaban_general}. Then
$$
\frac{1}{n} \bE_\Pi (\nrm{\widehat{\y} - \tilde{\y}}^2) \leq  \frac{1}{\dimsks} \nrm{ \solnls }^2_{\boldSigma + Tr(\boldSigma) \mathbf{I}_{p} + \kappa \sum_{i = 1}^{p} e_i B_i}
$$
where $\nrm{\mathbf{u}}_\mathbf{M} = \sqrt{\mathbf{u}^\top \mathbf{M} \mathbf{u} }$ is the Mahalanobis norm.
\label{lem:bound_diff_residuals}
\end{lem}
\begin{proof}
From Lemma \ref{lem:diff_residuals_ols}, we see that the difference between the two residuals is equal to the approximation error which results from fitting the model in the compressed domain. This approximation error measures the distance between the compressed space and the optimal regression function in the original space. It can be bounded as follows
$$
\nrm{\widehat{\y} - \tilde{\y}}^2 = \nrm{\At \solnls - \feats_k \solnls_k}^2 \leq \nrm{\At  \solnls  - \At \Pi \Pi^\top \solnls}^2.
$$
This bound follows from the fact that using $\At \Pi$ implies a larger degree of approximation than using $\feats_k$ which also contains raw features. In other words, the column space of  $\At \Pi$ is contained in the column space of $\feats_k$ and $\At$ and the distance between $\At \Pi$ and $\At$ is larger than the distance between $\feats_k$ and $\At$. Additionally, $\solnls_k$ minimizes the objective in the compressed space such that the coefficients given by $\Pi^\top \solnls $ cannot be associated with a smaller approximation error. 

Applying Lemma \ref{lem_kaban_general} yields the desired result.
\end{proof}

\begin{cor}
If $\RP$ is a Gaussian (or some other random projection with excess kurtosis, $\kappa=0$), then the bound in Lemma \ref{lem:bound_diff_residuals} reduces to
$$
\frac{1}{n} \bE_\Pi (\nrm{\widehat{\y} - \tilde{\y}}^2) \leq  \frac{1}{\dimsks} \nrm{ \solnls }^2_{\boldSigma + Tr(\boldSigma) \mathbf{I}_{p}}.
$$
\end{cor}

\begin{lem}
The $j$-th OLS regression coefficient is proportional to the inner product of a residual and the response, where the residual results from regressing the $j$-th feature vector onto the remaining part of the design matrix. More formally, let $\At^{j}$ be the feature vector of interest and let $\At^{-j}$ denote the matrix containing the remaining features. Then the residual $\widehat{\mathbf{z}}_j$ results from regressing $\At^{j}$ onto $\At^{-j}$
\begin{equation}
\widehat{\mathbf{z}}_j = \At^{j} - \At^{-j} \widehat{\gamma} \; \; \;  \text{   where   } \; \; \;  \widehat{\gamma} = \underset{\gamma}{\arg\min} \nrm{ \At^{j} - \At^{-j} \gamma}^2,
\label{eq:residual_ols_coeff}
\end{equation}
and the $j$-th regression coefficient can we written as
$$
\solnls_j = \frac{\widehat{\mathbf{z}}_j \cdot \y }{\widehat{\mathbf{z}}_j\cdot \At^{j}}.
$$ 
\label{lem:ols_coeff}
\end{lem}
\begin{proof}
This result reflects the fact that the $\solnls_j$ contains the additional contribution of feature $j$ on $\y$ after having accounted for the remaining features. The formula follows from ``regression by successive orthogonalization'', for more details see \citep{esl}.
\end{proof}

\paragraph{Proof of Theorem \ref{thm_ols}.}
From Lemma \ref{lem:ols_coeff}, we have a closed-form expression for the $j$-th OLS regression coefficient in the original space. Now consider the regression problem worker $k$ has to solve and let $j'$ be a raw feature in $\feats_k$ which corresponds to feature $j$ in $\At$. For ease of notation, set $j' = j$ such that we have $\At^{j} = \feats_k^{j} $. In the problem worker $k$ solves the expressions in eq. \eqref{eq:residual_ols_coeff} become
$$
\tilde{\mathbf{z}}_{j} = \At^{j}  - \feats_k^{-j} \tilde{\gamma} \; \; \; \text{   where   } \; \; \;  \tilde{\gamma} = \underset{\gamma_k}{\arg\min} \nrm{ \At^{j}  - \feats_k^{-j} \gamma_k}^2
\; \; \; \text{ and } \; \; \; 
\solnls_{k, j} = \frac{ \tilde{\mathbf{z}}_{j} \cdot \y }{\tilde{\mathbf{z}}_{j} \cdot \At^{j}}.
$$
The only change between the expressions for the $j$-th coefficient results from replacing $\widehat{\mathbf{z}}_j$ by $\tilde{\mathbf{z}}_{j}$. From Lemma \ref{lem:bound_diff_residuals} we have that the difference between the residuals $\widehat{\mathbf{z}}_j$ by $\tilde{\mathbf{z}}_{j}$ is bounded such that the difference between the estimated coefficients $\solnls_j$  and $\solnls_{k, j}$ is bounded as well. 


\addtocounter{si-sec}{1}
\section{Supporting results} \label{sec:supporting_lemma}

\subsection{Ridge Regression}
\begin{lem}[Lemma 1 of \citep{lu:2013} ] \label{lem:srht}
Let $\Wt$ be an $\dims\times r$  $(\dims > r)$ matrix where $\Wt\tr\Wt =\It_{r}$. Let $\RP$ be a $\dims\times\dimss$ SRHT matrix where $\dimss$ is the subsampling size and $\dims > \dimss > r$. Then with failure probability at most $\delta + \dims/e^{r}$
$$
\nrm{\Wt\tr\RP \RP\tr\Wt - \Wt\tr\Wt} \leq \sqrt{\frac{c\log(2r/ \delta)r}{\dimss}}.
$$
\end{lem}

We rely on the following theorem from \citep{lu:2013} which bounds the risk of the subsampled approximation to the ridge regression estimator. We provide an alternative proof of this result based on the bias-variance decomposition of the regularized kernel ridge regression estimator from \citep{Bach:2012vm}.
\begin{lem}[Risk Inflation \citep{lu:2013}] \label{lem:risk_inflation}
Consider solving ridge regression in the dual space with $\soln = \At\tr\boldalpha$, $\boldalpha\in \R ^\samp$ and letting $\Kt=\At\At\tr$ be the kernel matrix, the dual optimization problem is
$$
\widehat{\boldalpha} = \arg\min_{\boldalpha\in\R^\samp} \samp^{-1}\nrm{\y - \Kt\alpha}_2^2 + \lambda \boldalpha\tr\Kt\boldalpha
$$
with solution $\widehat{ \boldalpha } = \br{\Kt+\samp\lambda\It}^{-1}\y$ and $\solnrr  = \Xt\tr\widehat{\boldalpha}$. Using $\Kt_H =  \Xt\RP(\Xt\RP)\tr$ instead of  $\Kt = \Xt\Xt\tr$, where $\RP\in\R^{\dims\times\dimss}$ is a SRHT matrix, allows for a faster computation. Let $\solnrr_H $ denote the resulting estimate from this  randomized approximation and let $r$ be the rank of the $\Xt$ matrix. With probability at least $1 - (\delta + \dims/e^{r})$ we have the following relation between the risk of ridge regression and the randomized approximation
$$
\risk(\Xt\RP\solnrr_H) \leq (1-\rho)^{-2} \risk(\Xt\solnrr)
$$
where $\rho = C \sqrt{\frac{r \log(2r / \delta) }{\dimss}}$.
\end{lem}

\begin{proof}
We follow the proof of Theorem 1 in \citep{Bach:2012vm} and consider the regularized approximation to $\Kt$
$$
\Lt_\gamma = \Xt \Nt_\gamma \Xt\tr
$$
where $\gamma > 0 $ and
$$
\Nt_\gamma = \frac{1}{1 + \gamma} \RP \RP\tr
$$
where $\RP = \sqrt{\frac{\dims}{\dimss}} \Dt \Ht \St $ is the SRHT. With $\gamma = 0$, we have $\Nt_0 = \RP \RP\tr$ so $\Lt_0 = \Kt_H$ which is the quantity of interest. 
Lemma \ref{lem:srht} implies that with probability at least $1 - (\delta + \dims/e^{r})$
$$
(1 - \rho) \Kt \preceq \Lt_0 \preceq (1 + \rho) \Kt
$$
where $\rho = C \sqrt{\frac{r \log(2r / \delta) }{\dimss}}$.

Since $\Xt \solnrr = \Kt \widehat{\boldalpha}$, we have that $\risk(\Xt \solnrr) \equiv \risk(\Kt \widehat{\boldalpha})$. Setting $\z = \bE_\varepsilon \sq{\y} = \Xt \soln^*$ the risk can be decomposed in the following way \citep{Bach:2012vm} 
$$
\risk(\Xt\solnrr) = \risk(\Kt \widehat{\boldalpha}) = n^{-1} \bE_{\varepsilon} {\nrm{\z - \Kt \widehat{\boldalpha}}^2} 
= \underbrace{\frac{\sigma^2}{n} \trace{\left[ \Kt(\Kt + n \lambda\It)^{-1}    \right]^2 }}_{\var(\Kt)} + \underbrace{n \lambda^2 \z\tr(\Kt+ n \lambda \It)^{-2}\z}_{\bias(\Kt)}
$$
where $\var(\cdot)$ is the variance and $\bias(\cdot)$ is the bias.
When learning with $\Lt_\gamma$ instead of $\Kt$, the variance term of the risk is given by
\begin{equation}
\var(\Lt_\gamma) = \frac{\sigma^2}{n} \trace{\left[ \Lt_\gamma( \Lt_\gamma + n \lambda\It)^{-1}    \right]^2 }.
\label{variance}
\end{equation}
The function $\gamma \mapsto \Nt_\gamma$ is matrix-non-increasing (i.e., if $\gamma \geq \gamma'$ then $\Nt_\gamma \preceq \Nt_{\gamma'}$). Therefore, we have $0 \preceq \Nt_\gamma \preceq \Nt_0$. Since the variance $\var(\Lt_\gamma)$ is non-decreasing in $\Nt_\gamma$, this implies 
$\var(\Lt_\gamma)\leq\var(\Lt_0)$. 
Furthermore, as $\Lt_0 \preceq (1 + \rho) \Kt$
\begin{eqnarray*}
\Xt \RP \RP \tr \Xt \tr&\preceq& (1 + \rho) \Kt \\
\Nt_0  = \RP \RP \tr&\preceq& (1 + \rho) \It \\
\end{eqnarray*}
Thus, we have
$$
\var(\Kt_H) \leq \var((1 + \rho)\Kt).
$$
For the bias term, we have
\begin{equation}
\bias(\Lt_\gamma) = n \lambda^2 \z\tr(\Lt_\gamma + n \lambda \It)^{-2}\z
\label{bias}
\end{equation}
which is a non-decreasing function of $\gamma$.  
For $\gamma = 0$, $\Nt_0$ is lower-bounded as follows
\begin{eqnarray*}
\Xt \RP \RP \tr \Xt \tr&\succeq & (1 - \rho) \Kt \\
\Nt_0 = \RP \RP \tr&\succeq& (1 - \rho) \It \\
\end{eqnarray*}
and as the bias is non-increasing in $\Nt_\gamma$
$$
\bias(\Kt_H) \leq \bias((1-\rho) \Kt).
$$
Finally, the risk can be bounded as
\begin{eqnarray*}
\risk(\Xt\RP\solnrr_H) &=& \var(\K_H) + \bias(\K_H) \\
& \leq &  \var((1+\rho) \Kt) + \bias((1-\rho) \Kt) \\
& \leq & (1-\rho)^{-2}  \br{\var(\K) + \bias(\K)}  =  (1-\rho)^{-2} \risk(\Xt\solnrr).
\end{eqnarray*}
\end{proof}

\subsection{Ordinary Least Squares}
Here we state a result by \citep{kaban:2014} which we rely on to state Theorem \ref{thm_ols}. It bounds the approximation error that is incurred by fitting the model in the compressed space as opposed to the initial space. We defer the proof to the original paper.

\begin{lem}[Compressive least squares \citep{kaban:2014}]
Let $\Pi$ be a $p \times \dimss$ random matrix, $\dimss < p$, with entries drawn i.i.d. from a zero-mean symmetric distribution with variance $1/ \dimss$ and excess kurtosis $\kappa = \frac{\bE(\Pi^4_{ij})}{\bE(\Pi^2_{ij})^2} - 3 $. Let $\Sigma = \At^\top\At/n$ be fixed with eigenvalues $e_i, \ldots, e_p$ and let $B_i$ be a $p \times p$ diagonal matrix with the $j$-th diagonal element being $ \sum_{a = 1}^p U^2_{ai} U^2_{aj}$ where $U_{ai}$ is the $a$-th entry of the $i$-th eigenvector of $\Sigma$. Finally, let $\solnls$ contain the optimal regression coefficients in $\mathbb{R}^\dims$. Then
$$
\frac{1}{n}\bE_{ \RP}( \nrm{\At \solnls - \At \Pi \Pi^\top \solnls}^2) = \frac{1}{\dimss} \nrm{\solnls}^2_{\Sigma + \trace{\Sigma} \mathbf{I}_p + \kappa \sum_{i = 1}^p e_i B_i}
$$ 
where $\nrm{\mathbf{u}}_\mathbf{M} = \sqrt{\mathbf{u}^\top \mathbf{M} \mathbf{u} }$ is the Mahalanobis norm.
\label{lem_kaban_general}
\end{lem}
Note that for Gaussian random projections and sparse random projections $\kappa = 0$ while for random sign random projections \citep{achlioptas2003} $\kappa = -2$ such that the bound tightens \citep{kaban:2014}.


\subsection{Consequences of concatenating random projections.}
The following lemma might seem obvious to the reader but we provide it as confirmation of an intuitive result. The lemma is a minor reformulations of the row sampling lemma (Lemma 3.4) from \citep{Tropp:2010uo}. What this lemma confirms is that concatenating $(\blocks-1)$ lots of $\dimsks$ random projections as in {\bf Step 4.} of \loco is equivalent to computing a single, SRHT defined as $\RP\in\R^{(\dims\ -\dimsk)\times (\blocks-1)\dimsks}$. 

The consequence is that the computation of the SRHT can be divided among $\blocks$ workers and computed in parallel provided the $\dims$ coordinates are uniformly distributed among the workers. 

The proof is provided below and are very similar to the proof of the original lemma. 
\begin{lem}[Concatenated row sampling]\label{lem:concatenated_row}
Let $\Wt$ be an $\samp \times \dims$ matrix with orthonormal columns, and define the quantity $M:=\samp \cdot \max_{j=1,\ldots\samp}\nrm{e_j\tr\Wt}^2$. Let $\Wt_1,\ldots,\Wt_K$ be a balanced, random partitioning of the rows of $\Wt$ where each matrix $\Wt_k$ has exactly $\tau=\samp/ \blocks$ rows. For a positive parameter $\alpha$, select the subsample size 
$$
l \cdot \blocks \geq \alpha M\log(\dims) .
$$
Let $\St_{T_k}\in\R^{l\times \tau}$ denote the operation of uniformly at random sampling a subset, $T_k$ of the rows of $\Wt_k$ by sampling $l$ coordinates from $\cb{1,2,\ldots \tau}$ without replacement. Now denote $\St \Wt$ as the concatenation of the subsampled rows 
$$\sq{\br{\St_{T_1}\Wt_1}\tr , \ldots , \br{\St_{T_K}\Wt_K}\tr}\tr.$$ 
Then
$$
\sqrt{\frac{(1-\delta)l \cdot K }{\samp}} \leq \sigma_{\dims}(\St \Wt) 
\quad \text{ and }  \quad  
\sigma_{1}(\St \Wt) \leq \sqrt{\frac{(1+\eta)l  \cdot K}{\samp}} 
$$
with failure probability at most 
$$
\dims\cdot \sq{\frac{e^{-\delta}}{(1-\delta)^{1-\delta}}}^{\alpha \log \dims} 
+ 
\dims\cdot \sq{\frac{e^{\eta}}{(1+\eta)^{1+\eta}}}^{\alpha \log \dims}
$$
\end{lem}
\begin{proof}
Define $\wt\tr_j$ as the $j^{th}$ row of $\Wt$ and $M:=\samp\cdot\max_j \nrm{\wt_j}^2$. 
$$
\Gt := \br{\St\Wt}\tr\br{\St\Wt} = \sum_{k=1}^K\sum_{j\in T_k} \wt_j\wt_j\tr .
$$
We can consider $\Gt$ as a sum of $l\cdot \blocks$ random matrices 
$$
{\Xt^{(1)}_{1}},\ldots,{\Xt^{(K)}_{1}}  
,\ldots ,
{\Xt^{(1)}_l}
,\ldots,{\Xt^{(K)}_{l}}
$$ 
sampled uniformly at random without replacement from the family $\mathcal{X} := \cb{\wt_i\wt_i\tr: i=1,\ldots,\dimsk \cdot \blocks}$.

To use the matrix Chernoff bound in Lemma \ref{lem:chernoff}, we require the quantities $\mu_{\min}$, $\mu_{\max}$ and $B$. Noticing that $\lambda_{\max}(\wt_j\wt_j\tr) = \nrm{\wt_j}^2 \leq \frac{M}{\samp}$, we can set $B \leq M/\samp$.

Taking expectations with respect to the random partitioning ($\bE_{P}$) and the subsampling within each partition ($\bE_{S}$),
using the fact that columns of $\Wt$ are orthonormal we obtain 
$$
\bE \sq{{\Xt^{(k)}_{1}}} =  \bE_{P} \bE_{S} \Xt_1^{(k)} = \frac{1}{\blocks}  \frac{1}{\tau} \sum_{i=1}^{\blocks \tau} \wt_i\wt_i\tr = \frac{1 }{n}\Wt\tr\Wt = \frac{1 }{n}\It
$$
Recall that we take $l$ samples in $\blocks$ blocks so we can define
$$
\mu_{\min} = \frac{l\cdot K}{\samp} \qquad \text{ and } \qquad \mu_{\max} = \frac{l\cdot K}{\samp} .
$$
\\
Plugging these values into Lemma \ref{lem:chernoff}, the lower and upper Chernoff bounds respectively yield
\begin{align*}
\mathbb{P} \cb{ \lambda_{\min}\br{\Gt} \leq (1-\delta)\frac{l \cdot  \blocks}{\samp}} &  \leq \dims \cdot \sq{ \frac{e^{-\delta}}{(1-\delta)^{1-\delta}}}^{l\cdot \blocks/M} \text{ for } \delta \in [0,1), \text{ and}
\\
\\
\mathbb{P} \cb{ \lambda_{\max}\br{\Gt} \geq (1+\delta)\frac{l \cdot  \blocks}{\samp}} &  \leq \dims \cdot \sq{ \frac{e^{\delta}}{(1+\delta)^{1+\delta}}}^{l\cdot \blocks/M} \text{ for } \delta \geq 0 .
\end{align*}

Noting that $\lambda_{\min}(\Gt) = \sigma_{\dims}(\Gt)^2$, similarly for $\lambda_{\max}$ and using the identity for $\Gt$ above obtains the desired result.
\end{proof}

For ease of reference, we also restate the Matrix Chernoff bound from \citep{Tropp:2010uo,Tropp:2010vm} but defer its proof to the original papers.

\begin{lem}[Matrix Chernoff from \citep{Tropp:2010uo}] \label{lem:chernoff}
Let $\mathcal{X}$ be a finite set of positive-semidefinite matrices with dimension $\dims$, and suppose that
$$
\max_{\At\in\mathcal{X}} \lambda_{\max}(\At) \leq B
$$
Sample $\{ \At_1,\ldots,\At_l \}$ uniformly at random from $\mathcal{X}$ without replacement. Compute 
$$
\mu_{\min} = l \cdot \lambda_{\min}(\bE \Xt_1) \qquad \text{and } \qquad \mu_{\max} = l \cdot \lambda_{\max}(\bE \Xt_1) 
$$
Then
\begin{align*}
\mathbb{P} \cb{ \lambda_{\min}\br{\sum_i \At_i} \leq (1-\delta)\mu_{\min}} &  \leq \dims \cdot \sq{ \frac{e^{-\delta}}{(1-\delta)^{1-\delta}}}^{\mu_{\min}/B} \text{ for } \delta \in [0,1), \text{ and}
\\
\\
\mathbb{P} \cb{ \lambda_{\max}\br{\sum_i \At_i} \geq (1+\delta)\mu_{\max}} &  \leq \dims \cdot \sq{ \frac{e^{\delta}}{(1+\delta)^{1+\delta}}}^{\mu_{\max}/B} \text{ for } \delta \geq 0 .
\end{align*}
\end{lem}

\end{document}